\documentclass[twoside,11pt]{article}
\usepackage{authblk,amsmath,amsthm,style,amsfonts,amssymb,xspace,ulem,euscript,algorithm,algorithmic,,graphicx,xcolor,subfigure,url,wrapfig,mathrsfs}

\renewcommand{\emph}{\textit}
\newcommand{\R}{{R}}
\definecolor{darkblue}{rgb}{0.0,0.0,0.5}

\newcommand{\w}{{\boldsymbol{w}}}

\renewcommand{\a}{\boldsymbol{a}}
\renewcommand{\b}{\boldsymbol{b}}
\renewcommand{\u}{\boldsymbol{u}}

\renewcommand{\k}{k}

\newcommand{\D}{D}

\renewcommand{\H}{{H}}
\newcommand{\Hilb}{{\mathcal{H}}}
\renewcommand{\P}{{P}}
\renewcommand{\v}{\boldsymbol{v}}

\newcommand{\x}{{\boldsymbol{x}}}
\newcommand{\y}{\boldsymbol{y}}
\renewcommand{\a}{{\boldsymbol{a}}}
\newcommand{\E}{\mathbb{E}}
\newcommand{\X}{\mathcal{X}}
\newcommand{\loss}{l}
\newcommand{\vtheta}{{\boldsymbol{\theta}}}
\newcommand{\vsigma}{{\boldsymbol{\sigma}}}

\newcommand{\one}{\mathbf{1}}
\newcommand{\zero}{\mathbf{0}}
\newcommand{\smallspace}{{\hspace{2pt}}}
\newcommand{\verysmallspace}{{\hspace{2pt}}}

\newcommand{\HS}{{\rm HS\xspace}}
\newcommand{\spec}{{\rm spec\xspace}}
\newcommand{\norm}[1]{\left\Vert#1\right\Vert}

\newcommand{\iprod}[1]{\left\langle#1\right\rangle}
\newcommand{\abs}[1]{\left\vert#1\right\vert}

\newtheorem*{corollary*}{Corollary}

\DeclareMathOperator*{\argmin}{argmin\xspace}

\DeclareMathOperator*{\tr}{tr\xspace}

\newcommand{\paren}[1]{\left(#1\right)}

\newcommand{\inner}[1]{\left\langle#1\right\rangle}
\newcommand{\set}[1]{\left\{#1\right\}}
\newcommand{\cH}{{\mathcal{H}}}
\newcommand{\cX}{{\mathcal{X}}}

\newtheorem*{asspu}{Assumption {\bf (U)} (no-correlation)}

\begin{document} 
\title{\textbf{The Local Rademacher Complexity of $\ell_p$-Norm Multiple Kernel Learning}}
\author{\textbf{Marius Kloft}\footnote{A Part of the work was done while MK was at Learning Theory Group, Computer Science Division and Department of Statistics, University of California, Berkeley, CA 94720-1758, USA.}}
\affil{Machine Learning Laboratory\\Technische Universit\"at Berlin\\Franklinstr. 28/29\\10587 Berlin, Germany\\ {mkloft@mail.tu-berlin.de}}
\author{\textbf{Gilles Blanchard}}
\affil{Department of Mathematics\\   University of Potsdam\\    Am Neuen Palais 10\\  14469 Potsdam, Germany\\ {gilles.blanchard@math.uni-potsdam.de}}
\maketitle´


\begin{abstract}%
We derive an upper bound on the local Rademacher complexity of $\ell_p$-norm multiple kernel learning, which yields a tighter excess risk bound than global approaches. Previous local approaches aimed at analyzed the case $p=1$ only while our analysis covers all cases $1\leq p\leq\infty$, assuming the different feature mappings corresponding to the different kernels to be uncorrelated. 
We also show a lower bound that shows that the bound is tight,
and derive consequences regarding excess loss, namely
fast convergence rates of the order
$O(n^{-\frac{\alpha}{1+\alpha}})$, where $\alpha$ is the minimum eigenvalue decay rate of the individual kernels.

\end{abstract} 

\begin{keywords}
  multiple kernel learning, learning kernels, generalization bounds, local Rademacher complexity
\end{keywords}

\section{Introduction}

Propelled by the increasing ``industrialization'' of modern application domains such as bioinformatics or computer vision leading to the accumulation of vast amounts of data, the past decade experienced a rapid professionalization of machine learning methods. Sophisticated machine learning solutions such as the support vector machine can nowadays almost completely be applied out-of-the-box \citep{Weka}. Nevertheless, a displeasing stumbling block towards the complete automatization of machine learning remains that of finding the best abstraction or \emph{kernel} for a problem at hand.

In the current state of research,  there is little hope that a machine will be able to find automatically---or even engineer---the best kernel for a particular problem \citep{searle}. However, by restricting to a less general problem, namely to a finite set of base kernels the algorithm can pick from, one might hope to achieve automatic kernel selection: clearly, cross-validation based model selection \citep{Stone1974} can be applied if the number of base kernels is decent. Still, the performance of such an algorithm is limited by the performance of the best kernel in the set.

In the seminal work of \cite{LanCriGhaBarJor04} it was shown that it is computationally feasible to simultaneously learn a support vector machine \emph{and} a linear combination of kernels at the same time, if we require the so-formed kernel combinations to be positive definite and trace-norm normalized. Though feasible for small sample sizes, the computational burden of this so-called \emph{multiple kernel learning} (MKL) approach is still high. By further restricting the multi-kernel class to only contain convex combinations of kernels, the efficiency can be considerably improved, so that ten thousands of training points and thousands of kernels can be processed \citep{SonRaeSchSch06}.

However, these computational advances come at a price. Empirical evidence has accumulated showing that sparse-MKL optimized kernel combinations rarely help in practice and
frequently are to be outperformed by a regular SVM using an unweighted-sum kernel $K =\sum_m K_m$ \citep{WSNips,GehNow09},
leading for instance to the provocative question ``Can learning kernels help performance?''\citep{Cortes2009}. 

By imposing an $\ell_q$-norm, $q \geq 1$, rather than an $\ell_1$ penalty on the kernel combination coefficients, MKL was finally made useful for practical applications and profitable (Kloft et al., 2009, 2011)\nocite{KloBreSonZieLasMue09,KloBreSonZie2011}. The $\ell_q$-norm MKL is an empirical minimization algorithm that operates on the multi-kernel class consisting of functions $f:x\mapsto\langle\w,\phi_\k(x)\rangle$ with $\norm{\w}_{k}\leq D$, where $\phi_k$
is the kernel mapping into the reproducing kernel Hilbert space (RKHS) $\Hilb_k$ with kernel $k$ and norm $\norm{.}_k$, while the 
kernel $k$ itself ranges over the set of possible kernels $\big\{ \k=\sum_{m=1}^M\theta_m\k_m ~\Big|~  \Vert\vtheta\Vert_q\leq 1,~\vtheta\geq 0 \big\}$.  

In Figure~\ref{fig:tss}, we reproduce exemplary results taken from Kloft et al.~(2009, 2011) (see also references therein for further
evidence pointing in the same direction). We first observe that,
as expected, $\ell_q$-norm MKL enforces strong sparsity in the coefficients $\theta_m$ when $q=1$, and no sparsity at all for $q=\infty$, which corresponds to the SVM with an unweighted-sum kernel, while intermediate values of $q$ enforce different degrees of soft sparsity (understood as the steepness of the decrease of the ordered coefficients $\theta_m$). Crucially, the performance (as measured by the AUC criterion) is not monotonic as a function of $q$;
$q=1$ (sparse MKL) yields significantly worse performance than $q=\infty$ (regular SVM with sum kernel), but optimal performance is attained for some intermediate value of $q$. This is an empirical strong motivation to study theoretically the performance of $\ell_q$-MKL beyond the limiting cases $q=1$ or $q=\infty$.

\begin{figure}[t]
  \centering
  \hspace{-0.6cm}
  \includegraphics[width=0.6\textwidth]{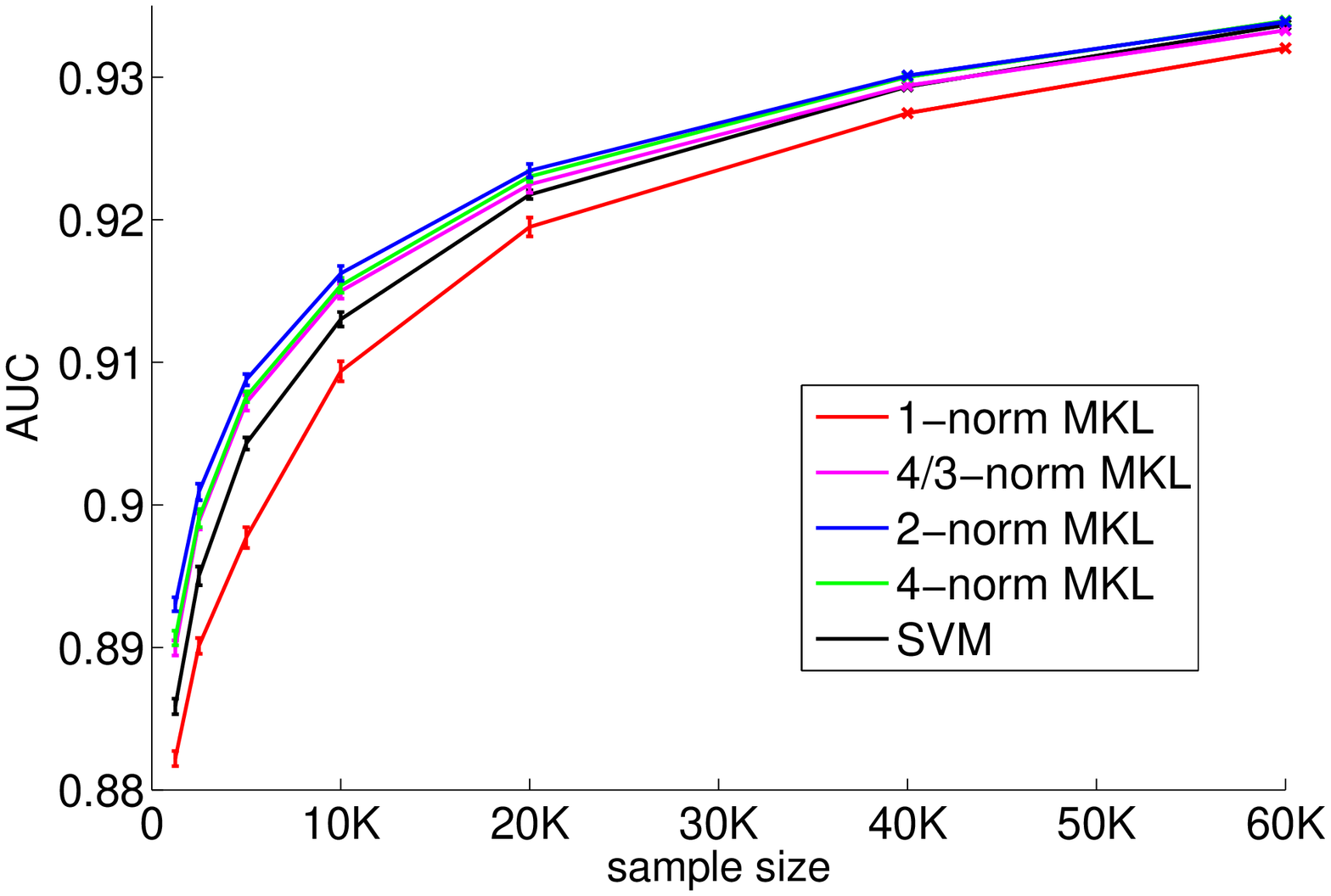}
  \hspace{-0.75cm}
  \includegraphics[width=0.45\textwidth]{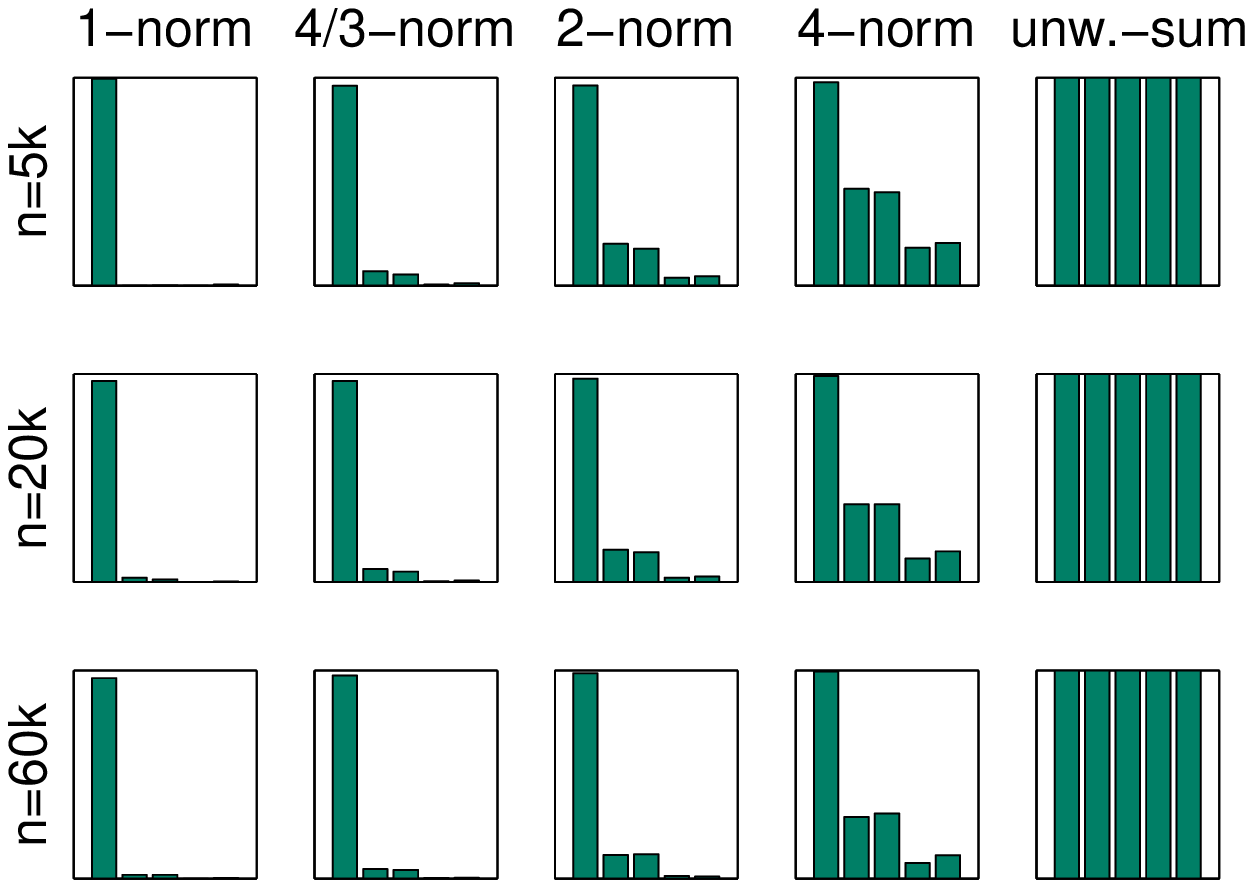}
  \caption{\label{fig:tss} Splice site detection experiment in Kloft et al. (2009, 2011). {\sc Left:} The Area under ROC curve 
    as a function of the training set size is shown. The regular SVM is equivalent to $q=\infty$ (or $p=2$). {\sc Right:} The optimal kernel weights $\theta_m$ as output by $\ell_q$-norm MKL are shown.}
\end{figure}

A conceptual milestone going back to the work of \cite{BacLanJor04} and \cite{MicPon05} is that the above multi-kernel class can equivalently be represented as a block-norm regularized linear class in the product Hilbert space $\Hilb:=\Hilb_1\times\cdots\times\Hilb_M$, where $\Hilb_m$ denotes the RKHS
associated to kernel $k_m$, $1\leq m \leq M$. More precisely, denoting by $\phi_m$ the kernel feature mapping associated to kernel $k_m$ over input space $\cX$, and 
$\phi : x \in \cX \mapsto (\phi_1(x),\ldots,\phi_M(x)) \in \Hilb$, the class of functions defined above coincides with
\begin{equation}
\label{eq:defclass}
\H_{p,\D,M} = \big\{f_{\w}:x\mapsto \inner{\w,\phi(x)} ~\big|~ \w=(\w^{(1)},\ldots,\w^{(M)}), \norm{\w}_{2,p}\leq \D \big\},
\end{equation}
where there is a one-to-one mapping of $q\in[1,\infty]$ to $p\in[1,2]$ given by $p=\frac{2q}{q+1}$. The $\ell_{2,p}$-norm is defined here as 
$ \big\Vert\w\big\Vert_{2,p}:=\big\Vert\big(\Vert\w^{(1)}\Vert_{k_1},\ldots,\Vert\w^{(M)}\Vert_{k_M}\big)\big\Vert_{p}=
\big(\sum_{m=1}^M \big\Vert \w^{(m)}\big\Vert_{k_m}^p \big)^{1/p}$; for simplicity, we will frequently write $\norm{\w^{(m)}}_2=\norm{\w^{(m)}}_{\k_m}$.

Clearly, learning the complexity of \eqref{eq:defclass} will be greater than one that is based on a single kernel only. However, it is unclear whether the increase is decent or considerably high and---since there is a free parameter $p$---how this relates to the choice of $p$. To this end the main aim of this paper is to analyze the sample complexity of the above hypothesis class \eqref{eq:defclass}. An analysis of this model, based on global Rademacher complexities, was developed by \cite{CorMohRos10}. In the present work, we base our main analysis on the theory of \emph{ local} Rademacher complexities, which allows to derive improved and more precise rates of convergence.

\paragraph{Outline of the contributions.}
This paper makes the following contributions:
\begin{itemize}
\item Upper bounds on the local Rademacher complexity of $\ell_p$-norm MKL are shown, from which we derive an excess risk bound that achieves a fast convergence rate of the order
$O(n^{-\frac{\alpha}{1+\alpha}})$, where $\alpha$ is the minimum eigenvalue decay rate of the individual kernels (previous bounds for $\ell_p$-norm MKL only achieved $O(n^{-\frac{1}{2}})$.
\item A lower bound is shown that beside absolute constants matches the upper bounds, showing that our results are tight.
\item The generalization performance of $\ell_p$-norm MKL as guaranteed by the excess risk bound is studied for varying values of $p$, shedding light on the appropriateness of a small/large $p$ in various learning scenarios.
\end{itemize}
Furthermore, we also present a simpler proof of the global Rademacher bound shown in \cite{CorMohRos10}. A comparison of the rates obtained with local and global Rademacher analysis, respectively, can be found in Section \ref{sec:disc-comp}.

\paragraph{Notation.}
For notational simplicity we will omit feature maps and directly view $\phi(x)$ and $\phi_m(x)$ as random variables $\x$ and $\x^{(m)}$ taking values in the Hilbert space $\cH$ and $\cH_m$, respectively, where  $\x=(\x^{(1)},\ldots,\x^{(M)})$.
Correspondingly, the hypothesis class we are interested in reads $ \H_{p,\D,M} = \big\{f_\w:\x\mapsto \inner{\w,\x} ~\big|~ \norm{\w}_{2,p}\leq \D \big\}.$ If $\D$ or $M$ are clear from the context, we sometimes synonymously denote $\H_p=\H_{p,\D}=\H_{p,\D,M}$. 
We will frequently use the notation $(\u^{(m)})_{m=1}^M$ for the element $\u=(\u^{(1)},\ldots,\u^{(M)}) \in \Hilb = \Hilb_1 \times \ldots \times \Hilb_M$.

We denote the kernel matrices corresponding to $\k$ and $\k_m$ by $K$ and $K_m$, respectively. Note that we are considering normalized kernel Gram matrices, i.e., the $ij$th entry of $K$ is $\frac{1}{n}k(\x_i,\x_j)$.
We will also work with covariance operators in Hilbert spaces. In a finite dimensional vector space, the (uncentered) covariance operator can be defined 
in usual vector/matrix notation as $\E \x\x^\top $. Since we are working with potentially infinite-dimensional vector spaces, we will use instead of 
$\x\x^\top$ the tensor notation $\x \otimes \x \in \HS(\Hilb)$, which is a Hilbert-Schmidt operator $\Hilb \mapsto \Hilb$ defined as  $(\x \otimes \x)\u = \inner{\x,\u}\x$. The space $\HS(\cH)$ of Hilbert-Schmidt operators on $\cH$  is itself a Hilbert space, and the expectation $\E \x \otimes \x$ is well-defined and belongs to $\HS(\cH)$
as soon as $\E \norm{\x}^2$ is finite, which will always be assumed (as a matter of fact, we will often assume that $\norm{\x}$ is bounded a.s.). We denote by $J=\E \x \otimes \x$, $J_m=\E \x^{(m)} \otimes \x^{(m)}$ the uncentered covariance operators corresponding to variables $\x$, $\x^{(m)}$; it holds that $\tr(J) = \E \norm{\x}^2_2$ and $\tr(J_m) = \E \norm{\x^{(m)}}^2_2$.

Finally, for $p\in[1,\infty]$ we use the standard notation $p^*$ to denote the conjugate of $p$, that is, $p^*\in[1,\infty]$ and $\frac{1}{p} + \frac{1}{p*} =1$.

\section{Global Rademacher Complexities in Multiple Kernel Learning}

We first review global Rademacher complexities (GRC) in multiple kernel learning. Let $\x_1,\ldots,\x_n$ be an i.i.d. sample drawn from $\P$. The global Rademacher complexity is defined as $\R(\H_p)=\E\sup_{f_\w\in\H_p} \langle\w,\frac{1}{n}\sum_{i=1}^n\sigma_i\x_i\rangle$, where $(\sigma_i)_{1\leq i\leq n}$ is an i.i.d. family (independent of $(\x_i)$ ) of Rademacher variables (random signs). Its empirical counterpart is denoted by
$\widehat{\R}(\H_p)=\E\big[\R(\H_p)\big|\x_1,\ldots,\x_n\big]=\E_\vsigma\sup_{f_\w\in\H_p} \langle\w,\frac{1}{n}\sum_{i=1}^n\sigma_i\x_i\rangle$. The interest in the global Rademacher complexity comes from that if known it can be used to bound the generalization error \citep{Kol01,BarMen02}. 

\smallskip \noindent In the recent paper of \cite{CorMohRos10} it was shown using a combinatorial argument that the empirical version of the global Rademacher complexity can be bounded as
$$ \widehat{\R}(\H_{p})\leq \D\smallspace\sqrt{\frac{cp^*}{n}\Big\Vert\big(\tr(K_m)\big)_{m=1}^M\Big\Vert_{\frac{p^*}{2}}} ,$$
where $c=\frac{23}{22}$ and $\tr(K)$ denotes the 
trace of the kernel matrix $K$. We will now show a quite short  proof of this result and then present a bound on the population version of the GRC.  The proof presented here is based on the Khintchine-Kahane inequality \citep{Kahane85} using the constants taken from Lemma 3.3.1 and Proposition 3.4.1 in \cite{KwaWoy92}. 

\begin{lemma}[Khintchine-Kahane inequality] \label{lemma:kahane}
Let be $\v_1,\ldots,\v_M\in\Hilb$. Then, for any $q\geq 1$, it holds
$$\E_\vsigma\big\Vert\sum_{i=1}^n\sigma_i\v_i\big\Vert_2^q
\leq  \Big(c\sum_{i=1}^{n}\big\Vert\v_i\big\Vert_2^2\Big)^{\frac{q}{2}},$$
where $c=\max(1,p^*-1)$. In particular the result holds for $c=p^*$.
\end{lemma}

\begin{proposition}[Global Rademacher complexity, empirical version]
For any $p\geq 1$ the empirical version of global Rademacher complexity 
of the multi-kernel class $\H_p$ can be bounded as 
$$ \forall t\geq p:\quad \widehat{\R}(\H_{p})\leq  \D\smallspace\sqrt{\frac{t^*}{n}\Big\Vert\big(\tr(K_m)\big)_{m=1}^M\Big\Vert_{\frac{t^*}{2}}} .$$
\end{proposition}
\begin{proof}
First note that it suffices to prove the result for $t=p$ as trivially $\norm{\x}_{2,t}\leq\norm{\x}_{2,p}$ holds for all $t\geq p$ and therefore $\R(H_p)\leq \R(H_t)$.
We can use a block-structured version of H\"older's inequality (cf. Lemma~\ref{prop-hoelder}) and the Khintchine-Kahane (K.-K.) inequality (cf. Lemma~\ref{lemma:kahane}) to bound the empirical version of the global Rademacher complexity as follows:
\begin{align*}
  \widehat{\R}(\H_p) &\stackrel{\text{def.}}{=} \E_\sigma\sup_{f_\w\in\H_p} \langle\w,\frac{1}{n}\sum_{i=1}^n\sigma_i\x_i\rangle \\
  &\stackrel{\text{H\"older}}{\leq} \D\smallspace\E_\vsigma\Big\Vert\frac{1}{n}\sum_{i=1}^n\sigma_i\x_i\Big\Vert_{2,{p^*}} \\
  &\stackrel{\text{Jensen}}{\leq}   \D\Big(\E_\vsigma\sum_{m=1}^M \Big\Vert\frac{1}{n}\sum_{i=1}^n\sigma_i\x_i^{(m)}\Big\Vert_{2}^{p^*}\Big)^{\frac{1}{p^*}} \\
  &\stackrel{\text{K.-K.}}{\leq}   \D\smallspace\sqrt{\frac{p^*}{n}}\Big(\sum_{m=1}^M \Big(\underbrace{\frac{1}{n}\sum_{i=1}^n\big\Vert\x_i^{(m)}\big\Vert^2_{2}}_{=\tr(K_m)}\Big)^{\frac{p^*}{2}}\Big)^{\frac{1}{p^*}} \\
  & =  \D\smallspace\sqrt{\frac{p^*}{n}\Big\Vert\big(\tr(K_m)\big)_{m=1}^M\Big\Vert_{\frac{p^*}{2}}},
\end{align*}
what was to show.
\end{proof}

\paragraph{Remark.}
Note that there is a very good reason to state the above bound in terms of $t\geq p$ instead of solely in terms of $p$: the Rademacher complexity $\widehat{\R}(\H_p)$ is not monotonic in $p$ and thus it is not always the best choice to take $t:=p$ in the above bound. This can is readily seen, for example, for the easy case where all kernels have the same trace---in that case
the bound translates into $\widehat{\R}(\H_{p})\leq\D\smallspace\sqrt{t^*M^{\frac{2}{t^*}}\frac{\tr(K_1)}{n}}$. Interestingly, the function $x\mapsto x M^{2/x}$ is not monotone and attains
its minimum for $x=2\log M$, where $\log$ denotes the natural logarithm with respect to the base $e$. This has interesting consequences: for any $p\leq (2\log M)^*$  we can take the bound $\widehat{\R}(\H_{p})\leq\D\smallspace\sqrt{\frac{e \log(M)\tr(K_1)}{n}}$, which has only a mild dependency on the number of kernels; note that in particular we can take this bound for the $\ell_1$-norm class $\widehat{\R}(\H_{1})$ for all $M>1$.

\bigskip\noindent Despite the simplicity the above proof, the constants are slightly better than the ones achieved in \cite{CorMohRos10}. However, computing the population version of the global Rademacher complexity of MKL is somewhat more involved and to the best of our knowledge has not  been addressed yet by the literature. To this end, note that from the previous proof we obtain $\R(\H_p)= \E\smallspace\D\sqrt{{p^*}/{n}}\big(\sum_{m=1}^M \big(\frac{1}{n}\sum_{i=1}^n\big\Vert\x_i^{(m)}\big\Vert^2_{\Hilb_m}\big)^{\frac{p^*}{2}}\big)^{\frac{1}{p^*}}$. We thus can use Jensen's inequality to move the expectation operator inside the root, 
\vspace{-0.25cm}
\begin{equation}\label{eq:aux-rosen}
  \R(\H_p)=\D\sqrt{{p^*}/{n}}\Big(\sum_{m=1}^M  \E\big(\frac{1}{n}\sum_{i=1}^n\big\Vert\x_i^{(m)}\big\Vert^2_2\big)^{\frac{p^*}{2}}\Big)^{\frac{1}{p^*}}\vspace{-0.2cm},
\end{equation}
but now need a handle on the $\frac{p^*}{2}$-th moments. To this aim we use the inequalities of \cite{Ros70} and Young \citep[e.g.,][]{SteeleBook} to show the following Lemma.

\begin{lemma}[Rosenthal + Young]\label{lemma:rosen}
Let $X_1,\ldots, X_n$ be independent nonnegative random variables satisfying $\forall i:X_i\leq B<\infty$ almost surely. Then, denoting $C_q = (2qe)^q$, for any $q\geq \frac{1}{2}$ it holds
$$\E\bigg(\frac{1}{n}\sum_{i=1}^nX_i\bigg)^q \leq C_q \bigg(\Big(\frac{B}{n}\Big)^q+\Big(\frac{1}{n}\sum_{i=1}^n\E X_i\Big)^q\bigg) .$$
\end{lemma}

\noindent The proof is defered to Appendix~\ref{app:lemmata}. It is now easy to show:

\begin{corollary}[Global Rademacher complexity, population version]
\label{cor:globrad}
Assume the kernels are  uniformly bounded, that is, $\norm{\k}_\infty\leq B<\infty$, almost surely. Then for any $p\geq 1$ the population version of global Rademacher complexity of the multi-kernel class $\H_p$ can be bounded as 
$$ \forall t\geq p:\quad \R(\H_{p,D,M})\leq \D\verysmallspace t^*\smallspace\sqrt{\frac{e}{n}\Big\Vert\big(\tr(J_m)
\big)_{m=1}^M\Big\Vert_{\frac{t^*}{2}}} 
      + \frac{\sqrt{Be}DM^{\frac{1}{t^*}}t^*}{n} .$$
For $t\geq 2$ the right-hand term can be discarded and the result also holds for unbounded kernels.
\end{corollary}

\begin{proof}
As above in the previous proof it suffices to prove the result for $t=p$. From \eqref{eq:aux-rosen} we conclude by the previous Lemma
\begin{align*}
   \R(\H_p) &\leq \D\smallspace \sqrt{\frac{p^*}{n}}\Bigg(\sum_{m=1}^M(ep^*)^{\frac{p^*}{2}}\bigg(\Big(\frac{B}{n}\Big)^{\frac{p^*}{2}}
      +\Big(\underbrace{\E\frac{1}{n}\sum_{i=1}^n\big\Vert\x_i^{(m)}\big\Vert^2_{2}}_{=\tr(J_m)}\Big)^{\frac{p^*}{2}}\bigg)\Bigg)^{\frac{1}{p^*}} \\
   &\leq \D p^*\smallspace\sqrt{\frac{e}{n}\Big\Vert\big(\tr(J_m)\big)_{m=1}^M\Big\Vert_{\frac{p^*}{2}}} 
      + \frac{\sqrt{Be}DM^{\frac{1}{p^*}}p^*}{n} ,
\end{align*}
where for the last inequality we use the subadditivity of the root function. Note that for $p\geq 2$ it is $p^*/2\leq 1$ and thus it suffices to employ Jensen's inequality instead of the previous lemma so that we come along without the last term on the right-hand side.
\end{proof}

\noindent For example, when the traces of the kernels are bounded,  the above bound is essentially determined by $O\Big(\frac{p^*M^{\frac{1}{p^*}}}{\sqrt{n}}\Big)$. We can also remark that by setting $t=(\log(M))^*$ we obtain the bound $\R(H_1)=O\Big(\frac{\log M}{\sqrt{n}}\Big)$.

\section{The Local Rademacher Complexity of Multiple Kernel Learning}\label{sec:LRC}

Let $\x_1,\ldots,\x_n$ be an i.i.d. sample drawn from $\P$. We define the local Rademacher complexity of $\H_p$  as 
$ \R_r(\H_p)=\E\sup_{f_\w\in\H_p:\P f_\w^2\leq r} \langle\w,\frac{1}{n}\sum_{i=1}^n\sigma_i\x_i\rangle$, where $\P f_\w^2 := \E(f_\w(\x))^2$.
Note that it subsumes the global RC  as a special case for $r=\infty$. 
As self-adjoint, positive Hilbert-Schmidt operators, covariance operators enjoy discrete eigenvalue-eigenvector decompositions $J=\E\x\otimes\x=\sum_{j=1}^\infty\lambda_j \u_j\otimes\u_j$ and $J_m=\E\x^{(m)}\otimes\x^{(m)}=\sum_{j=1}^\infty\lambda_j^{(m)} \u_j^{(m)}\otimes{\u_j^{(m)}}$, where $(\u_j)_{j\geq 1}$ and $(\u_j^{(m)})_{j\geq1}$
form orthonormal bases of $\Hilb$ and $\Hilb_m$, respectively.

We will need the following assumption for the case $1\leq p \leq 2$:
\begin{asspu}
\label{asp:uncorr}
The Hilbert space valued 
variables $\x_1,\ldots,\x_m$ are said to be (pairwise) uncorrelated if
for any $m\neq m'$ and $\w \in \Hilb_m\,, \w' \in \Hilb_{m'}$\,, the real
variables $\inner{\w,\x_m}$ and $\inner{\w',\x_{m'}}$ are uncorrelated.
\end{asspu}
Since $\Hilb_m,\Hilb_{m'}$ are RKHSs with kernels $k_m,k_{m'}$,
if we go back to the input random 
variable in the original space $X \in \X$, 
the above property is equivalent to saying that for any fixed $t,t' \in \X$, the
variables $k_m(X,t)$ and $k_{m'}(X,t')$ are uncorrelated.
This is the case, for example, if the original input space $\cX$  is $\mathbb{R}^M$, the orginal input variable $X\in \cX$ has
independent coordinates, and the kernels $k_1,\ldots,k_M$
each act on a different coordinate. Such a setting was considered
in particular by \cite{RasWaiYu10} in the setting of $\ell_1$-penalized MKL. We discuss this assumption in more detail in
Section \ref{sec:uass}.


We are now equipped to state our main results:

\begin{theorem}[Local Rademacher complexity, \mbox{$p \in [1,2]$} ]\label{thm:LRC_bound}
Assume that the kernels are uniformly bounded ($\norm{\k}_\infty\leq B<\infty$) and that Assumption {\bf (U)} holds. 
The local Rademacher complexity of the multi-kernel class $\H_p$ can be bounded  for any $1\leq p\leq 2$ as
\begin{align*}
   \forall t \in[p,2]:  \quad \R_r(\H_p) \leq
    \sqrt{\frac{16}{n}\bigg\Vert\bigg(\sum_{j=1}^\infty\min\Big(rM^{1-\frac{2}{t^*}},
ce\D^2{t^*}^2\lambda^{(m)}_j\Big)\bigg)_{m=1}^M\bigg\Vert_{\frac{t^*}{2}}}
       + \frac{\sqrt{Be}DM^{\frac{1}{t^*}}t^*}{n}.
\end{align*}
 \end{theorem}

\begin{theorem}[Local Rademacher complexity, $p\geq2$]\label{thm:LRC_bound_pgeq2}
The local Rademacher complexity of the multi-kernel class $\H_p$ can be bounded  for any $p\geq 2$ as
\[ \R_r(\H_p) \leq  \sqrt{\frac{2}{n}\sum_{j=1}^\infty\min(r, D^2 M^{\frac{2}{p^*}-1}\lambda_j)}.\]
\end{theorem}


\paragraph{Remark~1.}
Note that for the case $p=1$, by using $t=(\log(M))^*$\, 
in Theorem \ref{thm:LRC_bound},
we obtain the bound 
$$ \R_r(\H_1) \leq  \smallspace \sqrt{\frac{16}{n}\bigg\Vert\bigg(\sum_{j=1}^\infty\min\Big(rM,e^3D^2{(\log M)^2}\lambda^{(m)}_j\Big)\bigg)_{m=1}^M\bigg\Vert_\infty} 
      +\frac{\sqrt{B}e^{\frac{3}{2}}D\log(M)}{n}  .$$
\medskip
(See below after the proof of Theorem \ref{thm:LRC_bound} for a detailed justification.)

\paragraph{Remark~2.}
The result of Theorem \ref{thm:LRC_bound_pgeq2} for $p\geq 2$ 
can be proved using considerably simpler techniques and without imposing assumptions on boundedness nor on uncorrelation of the kernels. 
If in addition the variables $(\x^{(m)})$ 
are centered and uncorrelated, 
then the spectra are related as follows :
$~\spec(J)=\bigcup_{m=1}^M\spec(J_m)$; that is,
$\set{\lambda_i, i\geq 1}=\bigcup_{m=1}^M\set{\lambda_i^{(m)}, i \geq 1}$. 
Then one can write equivalently the bound of Theorem \ref{thm:LRC_bound_pgeq2} as 
$ \R_r(\H_p) \leq  \sqrt{\frac{2}{n}\sum_{m=1}^M\sum_{j=1}^\infty\min(r, D^2M^{\frac{2}{p^*}-1}\lambda_j^{(m)})} =  \sqrt{\frac{2}{n} \Big\Vert\Big(\sum_{j=1}^\infty\min(r,D^2 M^{\frac{2}{p^*}-1}\lambda_j^{(m)})\Big)_{m=1}^M\Big\Vert_1}$\,.
However, the main intended focus of this paper is on the more challenging case $1\leq p \leq 2$ which is usually studied in multiple kernel learning and relevant in practice.

\medskip

\paragraph{Remark~3.}
It is interesting to compare the above bounds for the special case $p=2$ with the ones of \cite{BarBouMen05}.
The main term of the 
bound of Theorem \ref{thm:LRC_bound_pgeq2} (taking $t=p=2$) 
is then essentially determined by  $O\Big(\sqrt{\frac{1}{n}\sum_{m=1}^M\sum_{j=1}^\infty\min\big(r,\lambda_j^{(m)}\big)}\Big)$.  If the variables $(\x^{(m)})$ are centered and uncorrelated, 
by the relation between the spectra stated in Remark 2, this is equivalently of order
$O\Big(\sqrt{\frac{1}{n}\sum_{j=1}^\infty\min\big(r,\lambda_j\big)}\Big)$, which is also what we obtain through Theorem \ref{thm:LRC_bound_pgeq2}, and coincides 
with the rate shown in \cite{BarBouMen05}.
\medskip

\begin{proof}\textbf{of Theorem~\ref{thm:LRC_bound} and Remark~1.~}
The proof is based on first relating the complexity of the class $\H_p$ with its centered counterpart, i.e., where all functions $f_\w\in\H_p$ are centered around their expected value. Then we compute the complexity of the centered class by decomposing the complexity into blocks, applying the no-correlation assumption, and using the inequalities of H\"older and Rosenthal. Then we relate it back to the original class, which we in the final step relate to a bound involving the truncation of the particular spectra of the kernels. 
Note that it suffices to prove the result for $t=p$ as trivially $\R(H_p)\leq \R(H_t)$ for all $p\leq t$.

\medskip \noindent \textsc{Step 1: Relating the original class with the centered class.} \quad
In order to exploit the no-correlation assumption, we will work in large parts of the proof with the centered class $\tilde{\H_p}=\big\{ \tilde{f}_\w ~\big|~\Vert\w\Vert_{2,p}\leq D\big\}$, wherein $\tilde{f}_\w:\x\mapsto\langle\w,\tilde{\x}\rangle$, and $\tilde{\x}:=\x-\E\x$. We start the proof by noting that $\tilde{f}_\w(\x) = f_\w(\x) - \inner{\w,\E\x} = f_\w(\x) - \E \inner{\w,\x} = f_\w(\x) - \E f_\w(\x)$, so that, by the bias-variance decomposition, it holds that
\begin{equation}\label{eq:var_decomp}
  \P f_\w^2 =  \E f_\w(\x)^2 = \E \paren{ f_\w(\x)- \E f_w(\x)}^2
+ \paren{\E f_w(\x)}^2 = \P\tilde{f}_\w^2 ~+~ \big(\P f_\w\big)^2 \,.
\end{equation}
Furthermore we note that by Jensen's inequality
\begin{align}
  \big\Vert\E\x\big\Vert_{2,p^*} &= \bigg(\sum_{m=1}^M\big\Vert\E\x^{(m)}\big\Vert^{p^*}_2\bigg)^{\frac{1}{p^*}} 
      = \bigg(\sum_{m=1}^M\big\langle\E\x^{(m)},\E\x^{(m)}\big\rangle^{\frac{p^*}{2}}\bigg)^{\frac{1}{p^*}} \nonumber \\
  &\stackrel{\text{Jensen}}{\leq}  \bigg(\sum_{m=1}^M\E\big\langle\x^{(m)},\x^{(m)}\big\rangle^{\frac{p^*}{2}}\bigg)^{\frac{1}{p^*}} 
      ~=~ \sqrt{\Big\Vert \Big(\tr(J_m)\Big)_{m=1}^M\Big\Vert_{\frac{p^*}{2}}} 
\label{eq:Ex_aux}
\end{align}
so that we can express the complexity of the centered class in terms of the uncentered one as follows:
\begin{align*}
  \R_r(\H_p) &= \E\sup_{\substack{f_\w\in\H_p,\\\P f_\w^2\leq r}} \big\langle\w,\frac{1}{n}\sum_{i=1}^n\sigma_i\x_i\big\rangle\\
  &\leq \E\sup_{\substack{f_\w\in\H_p,\\\P f_\w^2\leq r}}  \big\langle\w,\frac{1}{n}\sum_{i=1}^n\sigma_i\tilde{\x_i}\big\rangle ~+ \E\sup_{\substack{f_\w\in\H_p,\\\P f_\w^2\leq r}} \big\langle\w,\frac{1}{n}\sum_{i=1}^n\sigma_i\E\x\big\rangle
\end{align*}
Concerning the first term of the above upper bound, 
using \eqref{eq:var_decomp} we have
$P\tilde{f}_\w^2 \leq P f_\w^2$\,, and thus
\[
\E\sup_{\substack{f_\w\in\H_p,\\\P f_\w^2\leq r}}  \big\langle\w,\frac{1}{n}\sum_{i=1}^n\sigma_i\tilde{\x_i}\big\rangle
 \leq \E\sup_{\substack{f_\w\in\H_p,\\\P \tilde{f}_\w^2\leq r}}  \big\langle\w,\frac{1}{n}\sum_{i=1}^n\sigma_i\tilde{\x_i}\big\rangle
= \R_r(\tilde{\H}_p).
\]
Now to bound the second term, we write
\begin{align*}
\E\sup_{\substack{f_\w\in\H_p,\\\P f_\w^2\leq r}} \big\langle\w,\frac{1}{n}\sum_{i=1}^n\sigma_i\E\x\big\rangle 
  & =  \E \abs{\frac{1}{n}\sum_{i=1}^n\sigma_i}
  \sup_{\substack{f_\w\in\H_p,\\\P f_\w^2\leq r}} \inner{\w,\E\x}\\
& \leq \sup_{\substack{f_\w\in\H_p,\\\P f_\w^2\leq r}} \big\langle\w,\E\x\big\rangle
\paren{\E\paren{\frac{1}{n}\sum_{i=1}^n\sigma_i}^2}^{\frac{1}{2}}\\
& = \sqrt{n} \sup_{\substack{f_\w\in\H_p,\\\P f_\w^2\leq r}} \inner{\w,\E\x}.
\end{align*}
Now observe finally that we have 
\[
\inner{\w,\E\x} \stackrel{\text{H\"older}}{\leq} \norm{\w}_{2,p}\norm{\E\x}_{2,p^*}
\stackrel{\eqref{eq:Ex_aux}}{\leq}
\norm{\w}_{2,p} \smallspace\sqrt{\big\Vert \big(\tr(J_m)\big)_{m=1}^M\big\Vert_{\frac{p^*}{2}}}\]
as well as
\[
\inner{\w,\E\x} = \E f_\w(\x) \leq \sqrt{P f_\w^2}.
\]
We finally obtain, putting together the steps above,
\begin{equation}\label{eq:cent_uncent}
 \R_r(\H_p) \leq  \R_r(\tilde{\H}_p)+  n^{-\frac{1}{2}}\min\Big(\sqrt{r},D\smallspace\sqrt{\big\Vert \big(\tr(J_m)\big)_{m=1}^M\big\Vert_{\frac{p^*}{2}}} \Big)
\end{equation}
This shows that we at the expense of the additional summand on the right hand side we can work with the centered class instead of the uncentered one.
 
\medskip \noindent \textsc{Step 2: Bounding the complexity of the centered class.} \quad Since the (centered) covariance operator $\E\tilde{\x}^{(m)}\otimes\tilde{\x}^{(m)}$ is also a self-adjoint Hilbert-Schmidt operator on $\Hilb_m$, there exists an eigendecomposition
\begin{equation}\label{eq:eigen}
  \E\tilde{\x}^{(m)}\otimes\tilde{\x}^{(m)}=\sum_{j=1}^\infty \tilde{\lambda}^{(m)}_j\tilde{\u}^{(m)}_j\otimes\tilde{\u}^{(m)}_j ,
\end{equation}
wherein $(\tilde{\u}_j^{(m)})_{j \geq 1}$ is an orthogonal basis of $\Hilb_m$. 
Furthermore, the no-correlation assumption {\bf(U)} entails $\E\tilde{\x}^{(l)}\otimes\tilde{\x}^{(m)}=\zero$ for all $l\neq m$. 
As a consequence,
\begin{eqnarray}\label{eq:pf-bound}
  \P \tilde{f}_\w^2 &=& \E (f_\w(\tilde{\x}))^2 ~=~ \E\Big(\sum_{m=1}^M\big\langle\w_m,\tilde{\x}^{(m)}\big\rangle\Big)^2 ~=~
      \sum_{l,m=1}^M\Big\langle\w_l,\big(\E\tilde{\x}^{(l)}\otimes\tilde{\x}^{(m)}\big)\w_m\Big\rangle \nonumber \\
  &\stackrel{\textrm{\bf(U)}}{=}&\sum_{m=1}^M\Big\langle\w_m,\big(\E\tilde{\x}^{(m)}\otimes\tilde{\x}^{(m)}\big)\w_m\Big\rangle  ~=~
      \sum_{m=1}^M\sum_{j=1}^\infty\tilde{\lambda}_j^{(m)}\iprod{\w_m,\tilde{\u}_j^{(m)}}^2
\end{eqnarray}
and, for all $j$ and $m$,
\begin{eqnarray}\label{eq:rad_bound}
  \E\Big\langle\frac{1}{n}\sum_{i=1}^n\sigma_i\tilde{\x}_i^{(m)},\tilde{\u}_j^{(m)}\Big\rangle^2 &=& \E\frac{1}{n^2}\sum_{i,l=1}^n\sigma_i\sigma_l\iprod{\tilde{\x}_i^{(m)},\tilde{\u}_j^{(m)}}\iprod{\tilde{\x}_l^{(m)},\tilde{\u}_j^{(m)}} 
  \stackrel{\sigma~\text{i.i.d.}}{=} \E\frac{1}{n^2}\sum_{i=1}^n \iprod{\tilde{\x}_i^{(m)},\tilde{\u}_j^{(m)}}^2 \nonumber\\
  &=& \frac{1}{n}\Big\langle\tilde{\u}_j^{(m)},\Big(\underbrace{\frac{1}{n}\sum_{i=1}^n\E\tilde{\x}_i^{(m)}\otimes\tilde{\x}_i^{(m)}}_{=\E\tilde{\x}^{(m)}\otimes\tilde{\x}^{(m)}}\Big)\tilde{\u}_j^{(m)}\Big\rangle 
  = \frac{\tilde{\lambda}_j^{(m)}}{n}.
\end{eqnarray}
Let now $h_1,\ldots,h_M$ be arbitrary nonnegative integers. We can express the local Rademacher complexity in terms of the eigendecompositon \eqref{eq:eigen} as follows
\begin{eqnarray*}
\R_r(\tilde{\H}_p)  &=&  \E\sup_{f_\w\in\tilde{\H}_p:\P \tilde{f}_\w^2\leq r} \Big\langle\w,\frac{1}{n}\sum_{i=1}^n\sigma_i\tilde{\x}_i\Big\rangle \\
  &=& \E\sup_{f_\w\in\tilde{\H}_p:\P \tilde{f}_\w^2\leq r} \Big\langle\big(\w^{(m)}\big)_{m=1}^M,\big(\frac{1}{n}\sum_{i=1}^n\sigma_i\tilde{\x}_i^{(m)}\big)_{m=1}^M\Big\rangle  \\
  &\leq& \E\sup_{\P \tilde{f}_\w^2\leq r} \Big\langle~\Big(\sum_{j=1}^{h_m}\sqrt{\tilde{\lambda}_j^{(m)}}\langle\w^{(m)},\tilde{\u}_j^{(m)}\rangle\tilde{\u}_j^{(m)}\Big)_{m=1}^M,\\
    && ~~\qquad\qquad\qquad\qquad\Big(\sum_{j=1}^{h_m}\sqrt{\tilde{\lambda}_j^{(m)}}^{-1}\langle\frac{1}{n}\sum_{i=1}^n
    \sigma_i\tilde{\x}_i^{(m)},\tilde{\u}_j^{(m)}\rangle\tilde{\u}_j^{(m)}\Big)_{m=1}^M~\Big\rangle \\
    && \qquad +~~\E\sup_{f_\w\in\tilde{\H}_p}\bigg\langle\w,\Big(\sum_{j=h_m+1}^\infty\langle\frac{1}{n}\sum_{i=1}^n
          \sigma_i\tilde{\x}_i^{(m)},\tilde{\u}_j^{(m)}\rangle\tilde{\u}_j^{(m)}\Big)_{m=1}^M\bigg\rangle \\
  &\stackrel{\text{C.-S.,~Jensen}}{\leq}& \sup_{\P \tilde{f}_\w^2\leq r}
      \Bigg[ \paren{\sum_{m=1}^M\sum_{j=1}^{h_m}\tilde{\lambda}_j^{(m)}\langle\w^{(m)},
      \tilde{\u}_j^{(m)}\rangle^2}^{\frac{1}{2}}\\
  && \qquad \qquad \qquad \qquad \times \paren{\sum_{m=1}^M\sum_{j=1}^{h_m}\paren{\tilde{\lambda}_j^{(m)}}^{-1}\E \big\langle \frac{1}{n}
       \sum_{i=1}^n\sigma_i\tilde{\x}_i^{(m)},\tilde{\u}_j^{(m)} \big\rangle^2}^{\frac{1}{2}}\Bigg]\\
  && \qquad  +~  \E\sup_{f_\w\in\tilde{\H}_p}\bigg\langle\w,\Big(\sum_{j=h_m+1}^\infty\langle\frac{1}{n}\sum_{i=1}^n
          \sigma_i\tilde{\x}_i^{(m)},\tilde{\u}_j^{(m)}\rangle\tilde{\u}_j^{(m)}\Big)_{m=1}^M\bigg\rangle \\
\end{eqnarray*}
so that \eqref{eq:pf-bound} and \eqref{eq:rad_bound} yield
\begin{eqnarray*}
\R_r(\tilde{\H}_p)  
  &\stackrel{\text{\eqref{eq:pf-bound},~\eqref{eq:rad_bound}}}{\leq}& \sqrt{\frac{r\sum_{m=1}^Mh_m}{n}} 
    + \E\sup_{f_\w\in\tilde{\H}_p}\bigg\langle\w,\Big(\sum_{j=h_m+1}^\infty\langle\frac{1}{n}\sum_{i=1}^n  \sigma_i\tilde{\x}_i^{(m)},\tilde{\u}_j^{(m)}\rangle\tilde{\u}_j^{(m)}\Big)_{m=1}^M\bigg\rangle\\
  &\stackrel{\text{H\"older}}{\leq}&  \sqrt{\frac{r\sum_{m=1}^Mh_m}{n}} + \D\smallspace\E\bigg\Vert\Big(\sum_{j=h_m+1}^\infty\langle\frac{1}{n}\sum_{i=1}^n  \sigma_i\tilde{\x}_i^{(m)},\tilde{\u}_j^{(m)}\rangle\tilde{\u}_j^{(m)}\Big)_{m=1}^M\bigg\Vert_{2,p^*} .
\end{eqnarray*}

\medskip \noindent \textsc{Step 3: Khintchine-Kahane's and Rosenthal's inequalities.} \quad
We can now use the Khintchine-Kahane (K.-K.) inequality (see Lemma~\ref{lemma:kahane} in Appendix~\ref{app:lemmata}) to further bound the right term in the above expression as follows
\begin{multline*}
\E\bigg\Vert\Big(\sum_{j=h_m+1}^\infty\langle\frac{1}{n}\sum_{i=1}^n  \sigma_i\tilde{\x}_i^{(m)},\tilde{\u}_j^{(m)}\rangle\tilde{\u}_j^{(m)}\Big)_{m=1}^M\bigg\Vert_{2,p^*}\\
\begin{aligned}
   &\stackrel{\text{Jensen}}{\leq} 
      \E\bigg(\sum_{m=1}^M\E_\vsigma\bigg\Vert\sum_{j=h_m+1}^\infty\langle\frac{1}{n}\sum_{i=1}^n 
      \sigma_i\tilde{\x}_i^{(m)},\tilde{\u}_j^{(m)}\rangle\tilde{\u}_j^{(m)}\bigg\Vert_{\Hilb_m}^{p^*}\bigg)^{\frac{1}{p^*}}\\
   &\stackrel{\text{K.-K.}}{\leq}
      \sqrt{\frac{p^*}{n}}~\E\bigg(\sum_{m=1}^M\Big(\sum_{j=h_m+1}^\infty\frac{1}{n}\sum_{i=1}^n\langle 
      \tilde{\x}_i^{(m)},\tilde{\u}_j^{(m)}\rangle^2\Big)^{\frac{p^*}{2}}\bigg)^{\frac{1}{p^*}} \\
   &\stackrel{\text{Jensen}}{\leq} \sqrt{\frac{p^*}{n}}~\bigg(\sum_{m=1}^M\E\Big(\sum_{j=h_m+1}^\infty\frac{1}{n}\sum_{i=1}^n\langle 
      \tilde{\x}_i^{(m)},\tilde{\u}_j^{(m)}\rangle^2\Big)^{\frac{p^*}{2}}\bigg)^{\frac{1}{p^*}},
\end{aligned}
\end{multline*}
Note that for $p\geq 2$ it holds that $p^*/2\leq 1$, and thus it suffices to employ Jensen's inequality once again in order to move the expectation operator inside the inner term.
In the general case we need a handle on the $\frac{p^*}{2}$-th moments and to this end employ Lemma~\ref{lemma:rosen} (Rosenthal + Young), which yields
\begin{multline*}
 \bigg(\sum_{m=1}^M\E\Big(\sum_{j=h_m+1}^\infty\frac{1}{n}\sum_{i=1}^n\langle 
      \tilde{\x}_i^{(m)},\tilde{\u}_j^{(m)}\rangle^2\Big)^{\frac{p^*}{2}}\bigg)^{\frac{1}{p^*}} \\
\begin{aligned}
   &\stackrel{\text{R+Y}}{\leq} 
   \Bigg(\sum_{m=1}^M ~(ep^*)^{\frac{p^*}{2}} ~\bigg(\Big(\frac{B}{n}\Big)^{\frac{p^*}{2}}+
    \Big(\sum_{j=h_m+1}^\infty\underbrace{\frac{1}{n}\sum_{i=1}^n\E\langle\tilde{\x}_i^{(m)},\tilde{\u}_j^{(m)}\rangle^2}_{=\tilde{\lambda}^{(m)}_j}\Big)^{\frac{p^*}{2}} \bigg)~\Bigg)^{\frac{1}{p^*}} \\
  &\stackrel{(*)}{\leq} \sqrt{ ep^* \Bigg(\frac{BM^{\frac{2}{p^*}}}{n} +\bigg(\sum_{m=1}^M \Big( \sum_{j=h_m+1}^\infty\tilde{\lambda}^{(m)}_j \Big)^{\frac{p^*}{2}}\bigg)^{\frac{2}{p^*}}\Bigg) } \\
  &= \sqrt{ ep^*\Bigg(\frac{BM^{\frac{2}{p^*}}}{n}+ \Bigg\Vert\bigg( \sum_{j=h_m+1}^\infty\tilde{\lambda}^{(m)}_j \bigg)^M_{m=1}\Bigg\Vert_{\frac{p^*}{2}} \Bigg)} \\
  &\leq \sqrt{ ep^*\Bigg(\frac{BM^{\frac{2}{p^*}}}{n}+ \Bigg\Vert\bigg( \sum_{j=h_m+1}^\infty\lambda^{(m)}_j \bigg)^M_{m=1}\Bigg\Vert_{\frac{p^*}{2}} \Bigg)}
\end{aligned} 
\end{multline*}      
where for $(*)$ we used the subadditivity of $\sqrt[p^*]{\cdot}$ and in the last step we applied the Lidskii-Mirsky-Wielandt theorem which gives $ \forall j,m: ~\tilde{\lambda}_j^{(m)}\leq \lambda_j^{(m)}$.
Thus by the subadditivity of the root function
\begin{eqnarray}\label{eq:prelim_bound}
  \R_r(\tilde{\H}_p) &\leq& \sqrt{\frac{r\sum_{m=1}^Mh_m}{n}} + \D\smallspace\sqrt{\frac{e{p^*}^2}{n} \Bigg(\frac{BM^{\frac{2}{p^*}}}{n}+\Bigg\Vert\bigg( \sum_{j=h_m+1}^\infty\lambda^{(m)}_j \bigg)^M_{m=1}\Bigg\Vert_{\frac{p^*}{2}} \Bigg)} \nonumber\\
  &=&  \sqrt{\frac{r\sum_{m=1}^Mh_m}{n}} + \sqrt{\frac{e{p^*}^2\D^2}{n} \Bigg\Vert\bigg( \sum_{j=h_m+1}^\infty\lambda^{(m)}_j \bigg)^M_{m=1}\Bigg\Vert_{\frac{p^*}{2}}} + \frac{\sqrt{Be}\D M^{\frac{1}{p^*}}p^*}{n} .
\end{eqnarray}

\medskip \noindent \textsc{Step 4: Bounding the complexity of the original class.} \quad
Now note that for all nonnegative integers $h_m$ we either have
$$n^{-\frac{1}{2}}\min\Big(\sqrt{r},D\smallspace\sqrt{\big\Vert \big(\tr(J_m)\big)_{m=1}^M\big\Vert_{\frac{p^*}{2}}}\Big)\leq\sqrt{\frac{e{p^*}^2\D^2}{n} \Big\Vert\bigg( \sum_{j=h_m+1}^\infty\lambda^{(m)}_j \bigg)^M_{m=1}\Bigg\Vert_{\frac{p^*}{2}}}$$
(in case all $h_m$ are zero) or it holds
$$n^{-\frac{1}{2}}\min\Big(\sqrt{r},D\smallspace\sqrt{\big\Vert \big(\tr(J_m)\big)_{m=1}^M\big\Vert_{\frac{p^*}{2}}}\Big) \leq \sqrt{\frac{r\sum_{m=1}^Mh_m}{n}}$$
(in case that at least one $h_m$ is nonzero) so that in any case we get 
\begin{multline}\label{eq:addterm}
   n^{-\frac{1}{2}}\min\Big(\sqrt{r},D\smallspace\sqrt{\big\Vert \big(\tr(J_m)\big)_{m=1}^M\big\Vert_{\frac{p^*}{2}}}\Big)\\
   \qquad\qquad\leq\sqrt{\frac{r\sum_{m=1}^Mh_m}{n}}+\sqrt{\frac{e{p^*}^2\D^2}{n} \Bigg\Vert\bigg( \sum_{j=h_m+1}^\infty\lambda^{(m)}_j \bigg)^M_{m=1}\Bigg\Vert_{\frac{p^*}{2}}}.
\end{multline}
Thus the following preliminary bound follows from \eqref{eq:cent_uncent} by \eqref{eq:prelim_bound} and \eqref{eq:addterm}:
\begin{eqnarray}\label{bound:needed-for-excess}
   \R_r({\H}_p)\leq \smallspace\sqrt{\frac{4r\sum_{m=1}^Mh_m}{n}} + \sqrt{\frac{4e{p^*}^2\D^2}{n} \Bigg\Vert\bigg( \sum_{j=h_m+1}^\infty\lambda^{(m)}_j \bigg)^M_{m=1}\Bigg\Vert_{\frac{p^*}{2}}} + \frac{\sqrt{Be}\D M^{\frac{1}{p^*}}p^*}{n} ,
\end{eqnarray}      
for all nonnegative integers $h_m\geq 0$.
We could stop here as the above bound is already the one that will be used in the subsequent section for the computation of the excess loss bounds. However, we can work a little more on the form of the above bound to gain more insight in the properties---we will show that it is related to the truncation of the spectra at the scale $r$.

\medskip \noindent \textsc{Step 5: Relating the bound to the  truncation of the spectra of the kernels.} \quad
To this end, notice that for all nonnegative real numbers $A_1,A_2$ and any $\a_1,\a_2\in\mathbb R^m_+$  it holds for all $q\geq 1$ 

\begin{eqnarray}
 \sqrt{A_1}+\sqrt{A_2}&\leq&\sqrt{2(A_1+A_2)} \label{eq-AB1}\\  
 \norm{\a_1}_q+\norm{\a_2}_q&\leq& 2^{1-{\frac{1}{q}}}\norm{\a_1+\a_2}_q\leq 2\norm{\a_1+\a_2}_q  \label{eq-AB2}
\end{eqnarray}
(the first statement follows from the concavity of the square root function and the second one is proved in appendix~\ref{app:lemmata}; see Lemma~\ref{lemma:combin}) and thus 
\medskip
\begin{eqnarray*}
  &&\hspace{-1.45cm}\R_r(\H_p) \\
\hspace{-2cm}    &\stackrel{\eqref{eq-AB1}}{\leq}& \sqrt{8\left(\frac{r\sum_{m=1}^Mh_m}{n} + 
      \frac{e{p^*}^2\D^2 }{n}\bigg\Vert\bigg(\sum_{j=h_m+1}^\infty\lambda^{(m)}_j\bigg)_{m=1}^M\bigg\Vert_{\frac{p^*}{2}}\right)} +\frac{\sqrt{Be}\D M^{\frac{1}{p^*}}p^*}{n}\\
  &\stackrel{\text{$\ell_1$-to-$\ell_{\frac{p^*}{2}}$}}{\leq}& \sqrt{\frac{8}{n}\left(rM^{1-\frac{2}{p^*}}\bigg\Vert \Big(h_m\Big)_{m=1}^M \bigg\Vert_{\frac{p^*}{2}} + 
      e{p^*}^2\D^2 \bigg\Vert\bigg(\sum_{j=h_m+1}^\infty\lambda^{(m)}_j\bigg)_{m=1}^M\bigg\Vert_{\frac{p^*}{2}}\right)} +\frac{\sqrt{Be}\D M^{\frac{1}{p^*}}p^*}{n}  \\
  &\stackrel{\eqref{eq-AB2}}{\leq} & \sqrt{\frac{16}{n}\bigg\Vert\bigg(rM^{1-\frac{2}{p^*}}h_m + e{p^*}^2\D^2
     \sum_{j=h_m+1}^\infty\lambda^{(m)}_j\bigg)_{m=1}^M\bigg\Vert_{\frac{p^*}{2}}}  +\frac{\sqrt{Be}\D M^{\frac{1}{p^*}}p^*}{n}  ,
\end{eqnarray*}  
where to obtain the second inequality we applied that for all non-negative $\a\in\mathbb R^M$ and $0<q<p\leq\infty$ it holds\footnote{We denote by $\a^q$ the vector with entries $a_i^q$ and by $\one$ the vector with entries all $1$.}
\begin{equation}\label{eq:conv}
   \hspace{-0.155cm}\text{\rm{($\ell_q$-to-$\ell_p$ conversion)}} ~ \quad \norm{\a}_q=\left\langle\one,\a^q\right\rangle^{\frac{1}{q}}\stackrel{\text{\rm
     H\"older}}{\leq}\left(\norm{\one}_{(p/q)^*}\norm{\a^q}_{p/q}\right)^{1/q} 
     = M^{\frac{1}{q}-\frac{1}{p}}\norm{\a}_p .~
\end{equation}
Since the above holds for all nonnegative integers $h_m$, it follows
\begin{eqnarray*} 
  \R_r(\H_p) &\leq& \sqrt{\frac{16}{n}\bigg\Vert\bigg(\min_{h_m\geq 0} ~rM^{1-\frac{2}{p^*}}h_m + e{p^*}^2\D^2 \sum_{j=h_m+1}^\infty\lambda^{(m)}_j
     \bigg)_{m=1}^M\bigg\Vert_{\frac{p^*}{2}}} +\frac{\sqrt{Be}\D M^{\frac{1}{p^*}}p^*}{n}  \\
  &=&  \sqrt{\frac{16}{n}\bigg\Vert\bigg(\sum_{j=1}^\infty\min\Big(rM^{1-\frac{2}{p^*}},e{p^*}^2\D^2\lambda^{(m)}_j\Big)\bigg)_{m=1}^M\bigg\Vert_{\frac{p^*}{2}}} +\frac{\sqrt{Be}\D M^{\frac{1}{p^*}}p^*}{n} ,
\end{eqnarray*}
which completes the proof of the theorem.


\medskip\noindent\textsc{Proof of the remark.}\quad
To see that Remark 1 holds notice that $\R(\H_1)\leq\R(\H_p)$ for all $p\geq 1$ and thus by choosing $p=(\log(M))^*$ the above bound implies
\begin{eqnarray*} 
  \R_r(\H_1) &\leq&   \sqrt{\frac{16}{n}\bigg\Vert\bigg(\sum_{j=1}^\infty\min\Big(rM^{1-\frac{2}{p^*}},e{p^*}^2\D^2\lambda^{(m)}_j\Big)
       \bigg)_{m=1}^M\bigg\Vert_{\frac{p^*}{2}}} +\frac{\sqrt{Be}\D M^{\frac{1}{p^*}}p^*}{n} \\
   &\stackrel{\ell_{\frac{p^*}{2}}-\text{to}-\ell_\infty}{\leq}&   \sqrt{\frac{16}{n}\bigg\Vert\bigg(\sum_{j=1}^\infty\min\Big(rM,e{p^*}^2M^{\frac{2}{p^*}}\D^2
        \lambda^{(m)}_j\Big)\bigg)_{m=1}^M\bigg\Vert_\infty}+\frac{\sqrt{Be}\D M^{\frac{1}{p^*}}p^*}{n} \\
   &=& \sqrt{\frac{16}{n}\bigg\Vert\bigg(\sum_{j=1}^\infty\min\Big(rM,e^3\D^2{(\log M)^2}\lambda^{(m)}_j\Big)\bigg)_{m=1}^M\bigg\Vert_\infty} +\frac{\sqrt{B}e^{\frac{3}{2}}\D(\log M)}{n} , 
\end{eqnarray*} 
which completes the proof.
\end{proof}

\begin{proof}\textbf{of Theorem \ref{thm:LRC_bound_pgeq2}.}

The eigendecomposition $\E\x\otimes\x=\sum_{j=1}^\infty \lambda_j\u_j\otimes\u_j$ yields
\begin{equation}\label{eq:pf-bound2}
  \P f_\w^2 = \E (f_\w(\x))^2 = \E \langle\w,\x\rangle^2 = \big\langle\w,(\E\x\otimes\x)\w\big\rangle = \sum_{j=1}^\infty\lambda_j\iprod{\w,\u_j}^2 ,
\end{equation}
and, for all $j$
\begin{eqnarray}\label{rad_bound1}
  \E\Big\langle\frac{1}{n}\sum_{i=1}^n\sigma_i\x_i,\u_j\Big\rangle^2 &=& \E\frac{1}{n^2}\sum_{i,l=1}^n\sigma_i\sigma_l\iprod{\x_i,\u_j}\iprod{\x_l,\u_j} 
  \stackrel{\sigma~\text{i.i.d.}}{=} \E\frac{1}{n^2}\sum_{i=1}^n \iprod{\x_i,\u_j}^2 \nonumber\\
  &=& \frac{1}{n}\Big\langle\u_j,\Big(\underbrace{\frac{1}{n}\sum_{i=1}^n\E\x_i\otimes\x_i}_{=\E\x\otimes\x}\Big)\u_j\Big\rangle
  = \frac{\lambda_j}{n}.
\end{eqnarray}
Therefore, we can use, for any nonnegative integer $h$, the Cauchy-Schwarz inequality and a block-structured version of H\"older's inequality (see Lemma~\ref{prop-hoelder}) to bound the local Rademacher complexity as follows:
\begin{eqnarray*}
  \R_r(\H_p) &=& \E\sup_{f_\w\in\H_p:\P f_\w^2\leq r} \big\langle\w,\frac{1}{n}\sum_{i=1}^n\sigma_i\x_i\big\rangle\\
    &=& \E\sup_{f_\w\in\H_p:\P f_\w^2\leq r} \big\langle\sum_{j=1}^h\sqrt{\lambda_j}\langle\w,\u_j\rangle\u_j,
      \sum_{j=1}^h\sqrt{\lambda_j}^{-1}\langle\frac{1}{n}\sum_{i=1}^n\sigma_i\x_i,\u_j\rangle\u_j\big\rangle \\
  &&\quad\quad\quad~~~ + ~ \big\langle\w,\sum_{j=h+1}^\infty\langle\frac{1}{n}\sum_{i=1}^n\sigma_i\x_i,\u_j\rangle\u_j\big\rangle \\
  &\stackrel{\text{C.-S.,}~\eqref{eq:pf-bound2},~\eqref{rad_bound1}}{\leq}& \sqrt{\frac{rh}{n}} + \E\sup_{f_\w\in\H_p} \big\langle\w,\sum_{j=h+1}^\infty\langle\frac{1}{n}\sum_{i=1}^n\sigma_i\x_i,\u_j\rangle\u_j\big\rangle \\
  &\stackrel{\text{H\"older}}{\leq}& \sqrt{\frac{rh}{n}} + D 
     \E \bigg\Vert\sum_{j=h+1}^\infty\langle\frac{1}{n}\sum_{i=1}^n\sigma_i\x_i,\u_j\rangle\u_j\bigg\Vert_{2,p^*} \\
  &\stackrel{\ell_{\frac{p^*}{2}}-\text{to}-\ell_2}{\leq}& \sqrt{\frac{rh}{n}} + D M^{\frac{1}{p^*}-\frac{1}{2}}\E \bigg\Vert\sum_{j=h+1}^\infty\langle\frac{1}{n}\sum_{i=1}^n\sigma_i\x_i,\u_j\rangle\u_j\bigg\Vert_{\Hilb} \\
  &\stackrel{\rm Jensen}{\leq}& \sqrt{\frac{rh}{n}} + D M^{\frac{1}{p^*}-\frac{1}{2}} \bigg(\sum_{j=h+1}^\infty\underbrace{\E\langle\frac{1}{n}\sum_{i=1}^n\sigma_i\x_i,\u_j\rangle^2}_{\stackrel{\eqref{rad_bound1}}{\leq}\frac{\lambda_j}{n}}\bigg)^{\frac{1}{2}} \\
  &\leq& \sqrt{\frac{rh}{n}} + \sqrt{\frac{D^2 M^{\frac{2}{p^*}-1}}{n}\sum_{j=h+1}^\infty\lambda_j} .
\end{eqnarray*}
Since the above holds for all $h$, the result now follows from $\sqrt{A}+\sqrt{B}\leq\sqrt{2(A+B)}$ for all nonnegative real numbers $A,B$ (which holds by the concavity of the square root function):
$$ \R_r(\H_p) \leq \sqrt{\frac{2}{n}
\min_{0\leq h\leq n}
\Big(rh+D^2 M^{\frac{2}{p^*}-1}\sum_{j=h+1}^\infty\lambda_j\Big)} = \sqrt{\frac{2}{n}\sum_{j=1}^\infty\min(r, D^2 M^{\frac{2}{p^*}-1}\lambda_j)} .$$
\end{proof}

\section{Lower Bound}

In this subsection we investigate the tightness of our bound on the local Rademacher complexity of $\H_p$. To derive a lower bound we consider the particular case where variables $\x^{(1)},\ldots,\x^{(M)}$ are i.i.d. For example, this happens if
the original input space $\cX$ is $\mathbb{R}^M$, the original input variable $X\in \cX$ has i.i.d. coordinates, 
and the kernels $k_1,\ldots,k_M$
are identical and each act on a different coordinate of $X$.
\begin{lemma}
Assume that the variables $\x^{(1)},\ldots,\x^{(M)}$ are centered and identically independently distributed. Then, the following lower bound holds for the local Rademacher complexity of $\H_p$ for any $p\geq 1$:
$$ \R_r(\H_{p,D,M}) ~ \geq ~ R_{rM}(\H_{1,DM^{{1}/{p^*}},1}) .$$
\end{lemma}

\begin{proof}
First note that since the $\x^{(i)}$ are centered and uncorrelated, that 
$$ \P f_\w^2 = \Big(\sum_{m=1}^M \big\langle\w_m,\x^{(m)}\big\rangle\Big)^2 =\sum_{m=1}^M \big\langle\w_m,\x^{(m)}\big\rangle^2.$$
Now it follows
\begin{eqnarray*}
   \R_r(\H_{p,D,M}) &=&  \E\sup_{\w:~\tiny \begin{matrix} \P f_\w^2\leq r\\ \norm{\w}_{2,p}\leq D\end{matrix}} \big\langle\w,\frac{1}{n}\sum_{i=1}^n \sigma_i\x_i\big\rangle \\
   &=& \E\sup_{\w:~\tiny \begin{matrix} \sum_{m=1}^M \big\langle\w^{(m)},\x^{(m)}\big\rangle^2\leq r\\ \norm{\w}_{2,p}\leq D\end{matrix}} \big\langle\w,\frac{1}{n}\sum_{i=1}^n \sigma_i\x_i\big\rangle \\
   &\geq& \E\sup_{\w:~\tiny \begin{matrix} \forall m:~ \big\langle\w^{(m)},\x^{(m)}\big\rangle^2\leq r/M\\ \norm{\w^{(m)}}_{2,p}\leq D\\ \norm{\w^{(1)}}=\cdots=\norm{\w^{(M)}}\end{matrix}} \big\langle\w,\frac{1}{n}\sum_{i=1}^n \sigma_i\x_i\big\rangle \\
  &=&  \E\sup_{\w:~\tiny \begin{matrix} \forall m:~ \big\langle\w^{(m)},\x^{(m)}\big\rangle^2\leq r/M \\ \forall m:~\norm{\w^{(m)}}_2\leq DM^{-\frac{1}{p}}\end{matrix}} \sum_{m=1}^M \big\langle\w^{(m)},\frac{1}{n}\sum_{i=1}^n \sigma_i\x_i^{(m)}\big\rangle \\
   &=& \sum_{m=1}^M \E\sup_{\w^{(m)}:~\tiny \begin{matrix} \big\langle\w^{(m)},\x^{(m)}\big\rangle^2\leq r/M \\ \norm{\w^{(m)}}_2\leq DM^{-\frac{1}{p}}\end{matrix}} \big\langle\w^{(m)},\frac{1}{n}\sum_{i=1}^n \sigma_i\x_i^{(m)}\big\rangle ~,
\end{eqnarray*}
so that we can use the i.i.d. assumption on $\x^{(m)}$ to equivalently rewrite the last term as
\begin{eqnarray*}
   \R_r(\H_{p,D,M}) 
   &\stackrel{\x^{(m)}~\text{i.i.d.}}{\geq}&  \E\sup_{\w^{(1)}:~\tiny \begin{matrix} \big\langle \w^{(1)},\x^{(1)}\big\rangle^2\leq r/M \\ \norm{\w^{(1)}}_2\leq DM^{-\frac{1}{p}}\end{matrix}} \big\langle M\w^{(1)},\frac{1}{n}\sum_{i=1}^n \sigma_i\x_i^{(1)}\big\rangle \\
   &=&  \E\sup_{\w^{(1)}:~\tiny \begin{matrix} \big\langle M\w^{(1)},\x^{(1)}\big\rangle^2\leq rM \\ \norm{M\w^{(1)}}_2\leq DM^{\frac{1}{p^*}}\end{matrix}} \big\langle M\w^{(1)},\frac{1}{n}\sum_{i=1}^n \sigma_i\x_i^{(1)}\big\rangle \\
   &=&  \E\sup_{\w^{(1)}:~\tiny \begin{matrix} \big\langle \w^{(1)},\x^{(1)}\big\rangle^2\leq rM \\ \norm{\w^{(1)}}_2\leq DM^{\frac{1}{p^*}}\end{matrix}} \big\langle \w^{(1)},\frac{1}{n}\sum_{i=1}^n \sigma_i\x_i^{(1)}\big\rangle \\
   &=&  R_{rM}(\H_{1,DM^{{1}/{p^*}},1})
\end{eqnarray*}
\end{proof}

\noindent In \cite{Men2003} it was shown that there is an absolute constant $c$ so that if $\lambda^{(1)}\geq \frac{1}{n}$ then for all $r\geq\frac{1}{n}$ it holds $R_r(\H_{1,1,1})\geq\sqrt{\frac{c}{n}\sum_{j=1}^\infty \min(r,\lambda_j^{(1)})}$. Closer inspection of the proof reveals that more generally it holds $R_r(\H_{1,D,1})\geq\sqrt{\frac{c}{n}\sum_{j=1}^\infty \min(r,D^2\lambda_j^{(1)})}$ if $\lambda_1^{(m)}\geq \frac{1}{nD^2}$ so that we can use that result together with the previous lemma to obtain:
\begin{theorem}[Lower bound]
  Assume that the kernels are centered and identically independently distributed. Then, the following lower bound holds for the local Rademacher complexity of $\H_p$. There is an absolute constant $c$ such that if $\lambda^{(1)}\geq \frac{1}{nD^2}$ then for all $r\geq\frac{1}{n}$ and $p\geq 1$,
\begin{equation}\label{eq:lower}
   \R_r(\H_{p,D,M}) ~ \geq ~   \sqrt{\frac{c}{n}\sum_{j=1}^\infty \min(rM,D^2M^{{2}/{p^*}}\lambda_j^{(1)})} .
\end{equation}
\end{theorem}

\noindent 
We would like to compare the above lower bound with the upper bound of Theorem~\ref{thm:LRC_bound}. 
To this end note that for centered identical independent kernels the upper bound reads
$$  \R_r(\H_p) \leq \sqrt{\frac{16}{n}\sum_{j=1}^\infty\min\Big(rM,ce\D^2{p^*}^2M^{\frac{2}{p^*}}\lambda^{(1)}_j\Big)} + \frac{\sqrt{Be}DM^{\frac{1}{p^*}}p^*}{n},  $$
which is of the order $O\big(\sqrt{\sum_{j=1}^\infty\min\big(rM,D^2M^{\frac{2}{p^*}}\lambda_j^{(1)}}\big)\big)$ and, disregarding the quickly converging term on the right hand side and absolute constants, again matches the upper bounds of the previous section. 
A similar comparison can be performed for the upper bound of Theorem~\ref{thm:LRC_bound_pgeq2}: by Remark 2 the bound reads
$$  \R_r(\H_p) \leq \sqrt{\frac{2}{n} \Big\Vert\Big(\sum_{j=1}^\infty\min(r, D^2M^{\frac{2}{p^*}-1}\lambda_j^{(m)})\Big)_{m=1}^M\Big\Vert_1} ,$$
which for i.i.d. kernels becomes $\sqrt{2/n\sum_{j=1}^\infty\min\big(rM,D^2M^{\frac{2}{p^*}}\lambda_j^{(1)}}\big)$ and thus, beside absolute constans, matches the lower bound.
This shows that the upper bounds of the previous section are tight.

\section{Excess Risk Bounds}

In this section we show an application of our results to prediction problems, such as classification or regression. To this aim, in addition to the data $\x_1,\ldots,\x_n$ introduced earlier in this paper, let also a label sequence $y_1,\ldots,y_n\subset[-1,1]$ be given that is i.i.d. generated from a probability distribution. The goal in statistical learning is to find a hypothesis $f$ from a pregiven class $\mathcal F$ that minimizes the expected loss $\E\verysmallspace\loss(f(\x),y)$, where $\loss:\mathbb R^2\mapsto[-1,1]$ is a predefined loss function that encodes the objective of given the learning/prediction task at hand. For example, the hinge loss $\loss(t,y)=\max(0,1-yt)$ and the squared loss $\loss(t,y)=(t-y)^2$ are frequently used in classification and regression problems, respectively.

Since the distribution generating the example/label pairs is unknown, the optimal decision function 
$$~f^* :=\argmin_f \E\verysmallspace\loss(f(\x),y)~$$
can not be computed directly and a frequently used method consists of instead minimizing the \emph{empirical} loss,
$$~\hat{f} :=\argmin_f \frac{1}{n}\sum_{i=1}^n\loss(f(\x_1),y_1).~$$ 
In order to evaluate the performance of this so-called \emph{empirical minimization} algorithm we study the excess loss,
$$ \P(\loss_{\hat{f}} - \loss_{f^*}) ~:=~ \E\verysmallspace\loss(\hat{f}(\x),y) - \E\verysmallspace\loss(f^*(\x),y) .$$
In \cite{BarBouMen05} and \cite{Kol06} it was shown that the rate of convergence of the excess risk is basically determined by the fixed point of the local Rademacher complexity. For example, the following result is a slight modification of Corollary 5.3 in \cite{BarBouMen05} that is well-taylored to the class studied in this paper.\footnote{We exploit the improved constants from Theorem 3.3 in \cite{BarBouMen05} because an absolute convex class is star-shaped. Compared to Corollary 5.3 in \cite{BarBouMen05} we also use a slightly more general function class ranging in $[a,b]$ instead of the interval $[-1,1]$. This is also justified by Theorem 3.3.}

\begin{lemma}\label{lemma:bartlett}
  Let $\mathcal F$ be an absolute convex class ranging in the interval $[a,b]$ and let $\loss$ be a Lipschitz continuous loss with constant $L$. Assume there is a positive constant $F$ such
  that $\forall f\in\mathcal F:~\P(f-f^*)^2\leq F\verysmallspace\P(\loss_f-\loss_{f^*})$. Then, denoting by $r^*$ the fixed point of 
  $$ 2FL\verysmallspace\R_{\frac{r}{4L^2}}(\mathcal F) $$
  for all $x>0$ with probability at least $1-e^{-x}$ the excess loss can be bounded as
  $$ \P(\loss_{\hat{f}} - \loss_{f^*}) \leq 7\frac{r^*}{F} + \frac{(11L(b-a)+27F)x}{n}.$$
\end{lemma}

\noindent The above result shows that in order to obtain an excess risk bound on the multi-kernel class $\H_p$ it suffices to compute the fixed point of our bound on the local Rademacher complexity presented in Section~\ref{sec:LRC}. To this end we show:

\begin{lemma}\label{cor:fixed-point}
Assume that $\norm{\k}_\infty\leq B$ almost surely and let $p\in[1,2]$. For the fixed point $r^*$ of the local Rademacher complexity $~2FL\R_{\frac{r}{4L^2}}(\H_p)~$ it holds
$$ r^*\leq\min_{0\leq h_m\leq\infty} \frac{4F^2\sum_{m=1}^M h_m}{n} +
   8FL\smallspace\sqrt{\frac{e{p^*}^2\D^2}{n}\bigg\Vert\bigg(\sum_{j=h_m+1}^\infty\lambda^{(m)}_j\bigg)_{m=1}^M\bigg\Vert_{\frac{p^*}{2}}} 
   +\frac{4\sqrt{Be}\D FLM^{\frac{1}{p^*}}p^*}{n} ~.$$ 
\end{lemma}

\begin{proof}
For this proof we make use of the bound \eqref{bound:needed-for-excess} on the local Rademacher complexity. Defining 
$$a=\frac{4F^2\sum_{m=1}^M h_m}{n} \quad\text{and}\quad b=4FL\smallspace\sqrt{\frac{e{p^*}^2\D^2}{n}\bigg\Vert\bigg(\sum_{j=h_m+1}^\infty\lambda^{(m)}_j\bigg)_{m=1}^M\bigg\Vert_{\frac{p^*}{2}}} +\frac{2\sqrt{Be}\D FL M^{\frac{1}{p^*}}p^*}{n} ~,$$ 
in order to find a fixed point of \eqref{bound:needed-for-excess} we need to solve for $r=\sqrt{ar} + b$, which is 
equivalent to  solving ${r}^2 - (a+2b)r + b^2=0$ for a positive root. Denote this solution by $r^*$.
It is then easy to see that $r^*\geq a + 2b$. Resubstituting the definitions of $a$ and $b$ yields the result.
\end{proof}

\noindent We now address the issue of computing actual rates of convergence of the fixed point $r^*$ under the assumption of algebraically decreasing eigenvalues of the kernel matrices, this means, we assume $\exists d_m: ~\lambda_j^{(m)}\leq d_mj^{-\alpha_m}$ for some $\alpha_m>1$. This is a common assumption and, for example, met for finite rank kernels and convolution kernels \citep{WilSmoSch01}. Notice that this implies
\begin{eqnarray}\label{eq:algebr_EVs}
   \sum_{j=h_m+1}^\infty \lambda^{(m)}_j &\leq& d_m \sum_{j=h_m+1}^\infty j^{-\alpha_m} ~\leq~ d_m\int_{h_m}^\infty x^{-\alpha_m}dx = 
     d_m\Big[\frac{1}{1-\alpha_m}x^{1-\alpha_m}\Big]_{h_m}^\infty \nonumber\\
     &=&  -\frac{d_m}{1-\alpha_m}h_m^{1-\alpha_m} ~.
\end{eqnarray}
To exploit the above fact, first note that by $\ell_p$-to-$\ell_q$ conversion
$$ \frac{4F^2\sum_{m=1}^M h_m}{n} ~\leq~ 4F\sqrt{\frac{F^2M\sum_{m=1}^M h_m^2}{n^2}} 
   ~\leq~ 4F\sqrt{\frac{F^2M^{2-\frac{2}{p^*}}\big\Vert\big(h_m^2)\big)_{m=1}^M\big\Vert_{{2}/{p^*}}}{n^2}}  $$
so that we can translate the result of the previous lemma by \eqref{eq-AB1}, \eqref{eq-AB2}, and \eqref{eq:conv} into 
\begin{eqnarray}\label{eq:fixed_point_comb}
   r^* &\leq&  \min_{0\leq h_m\leq\infty} 8F\smallspace\sqrt{\frac{1}{n}\bigg\Vert\bigg(\frac{F^2M^{2-\frac{2}{p^*}}
      h_m^2}{n}+4e{p^*}^2\D^2L^2\sum_{j=h_m+1}^\infty\lambda^{(m)}_j\bigg)_{m=1}^M\bigg\Vert_{\frac{p^*}{2}}} \nonumber\\ 
   && \quad\quad\quad + \frac{4\sqrt{Be}\D FLM^{\frac{1}{p^*}}p^*}{n} ~.
\end{eqnarray}
Inserting the result of \eqref{eq:algebr_EVs} into the above bound and setting the derivative with respect to $h_m$ to zero we find the optimal $h_m$ as
$$ h_m = \Big(4d_mep^*{}^2D^2F^{-2}L^2M^{\frac{2}{p^*}-2}n\Big)^{\frac{1}{1+\alpha_m}} .$$
Resubstituting the above into \eqref{eq:fixed_point_comb} we note that 
$$ r^* = O\Bigg(\sqrt{\Big\Vert\Big(n^{-\frac{2\alpha_m}{1+\alpha_m}}\Big)_{m=1}^M\Big\Vert_{\frac{p^*}{2}}}\verysmallspace\Bigg) $$
so that we observe that the asymptotic rate of convergence in $n$ is determined by the kernel with the smallest decreasing spectrum (i.e., smallest $\alpha_m$).
Denoting $d_{\max} := \max_{m=1,\ldots,M} d_m$, ~$\alpha_{\min} := \min_{m=1,\ldots,M} \alpha_m$, and $h_{\max} := \big(4d_{\max} ep^*{}^2D^2F^{-2}L^2M^{\frac{2}{p^*}-2}n\big)^{\frac{1}{1+\alpha_{\min}}}$ we can upper-bound \eqref{eq:fixed_point_comb} by
\begin{eqnarray}
   r^* &\leq& 8F\smallspace\sqrt{\frac{3-\alpha_m}{1-\alpha_m}F^2M^2h_{\max}^2n^{-2}} + \frac{4\sqrt{Be}\D FLM^{\frac{1}{p^*}}p^*}{n} \nonumber\\
   &\leq& 8\sqrt{\frac{3-\alpha_m}{1-\alpha_m}}F^2M h_{\max}n^{-1} + \frac{4\sqrt{Be}\D FLM^{\frac{1}{p^*}}p^*}{n} \nonumber\\
   &\leq& 16\sqrt{e\frac{3-\alpha_m}{1-\alpha_m}}(d_{\max}D^2F^{-2}L^2{p^*}^2)^{\frac{1}{1+\alpha_{\min}}}M^{1+\frac{2}{1+\alpha_{\min}}
     \big(\frac{1}{p^*}-1\big)}n^{-\frac{\alpha_{\min}}{1+\alpha_{\min}}} \nonumber\\
    &&  + \frac{4\sqrt{Be}\D FLM^{\frac{1}{p^*}}p^*}{n} \label{eq:improved_M-rate}~.
\end{eqnarray}
We have thus proved the following theorem, which follows by the above inequality, Lemma~\ref{lemma:bartlett}, and the fact that our class $\H_p$ ranges in $BDM^{\frac{1}{p^*}}$.

\begin{theorem}\label{theorem:excess}
Assume that $\norm{\k}_\infty\leq B$ and $\exists d_m: ~\lambda_j^{(m)}\leq d_mj^{-\alpha_m}$ for some $\alpha_m>1$. Let $\loss$ be a Lipschitz continuous loss with constant $L$ and assume there is a positive constant $F$ such
that $\forall f\in\mathcal F:~\P(f-f^*)^2\leq F\verysmallspace\P(\loss_f-\loss_{f^*})$. Then for all $x>0$ with probability at least $1-e^{-x}$ the excess loss 
of the multi-kernel class $\H_p$ can be bounded for $p\in[1,\ldots,2]$ as
\begin{eqnarray*}
   \P(\loss_{\hat{f}} - \loss_{f^*}) &\leq& \min_{t\in[p,2]}~~
       186\sqrt{\frac{3-\alpha_m}{1-\alpha_m}}\verysmallspace\big(d_{\max}D^2F^{-2}L^2{t^*}^2\big)^{\frac{1}{1+\alpha_{\min}}}M^{1+\frac{2}{1+\alpha_{\min}}\big(\frac{1}{t^*}-1\big)}n^{-\frac{\alpha_{\min}}{1+\alpha_{\min}}} \\
   && \qquad+ \frac{47\sqrt{B}\D LM^{\frac{1}{t^*}}t^*}{n} + \frac{(22BDLM^{\frac{1}{t^*}}+27F)x}{n}
\end{eqnarray*}
\end{theorem}

\noindent We see from the above bound that convergence can be almost as slow as $O\big(p^*M^{\frac{1}{p^*}}n^{-\frac{1}{2}}\big)$~ (if at least one $\alpha_m\approx 1$ is small and thus $\alpha_{\min}$ is small) and almost as fast as $O\big(n^{-1}\big)$~ (if $\alpha_m$ is large for all $m$ and thus $\alpha_{\min}$ is large). For example, the latter is the case if all kernels have finite rank and also the convolution kernel is an example of this type. 

Notice that we of course could repeat the above discussion to obtain excess risk bounds for the case $p\geq 2$ as well, but since it is very questionable that this will lead to new insights, it is omitted for simplicity.

\section{Discussion}

In this section we compare the obtained local Rademacher bound with the global one, discuss related work as well as the assumption {\bf (U)}, and give a practical application of the bounds by studying the appropriateness of small/large $p$ in various learning scenarios.

\subsection{Global vs. Local Rademacher Bounds}\label{sec:disc-comp}

In this section, we discuss the rates obtained from the bound in Theorem \ref{theorem:excess} for the excess risk and compare them to the rates obtained using the global Rademacher complexity bound of Corollary~\ref{cor:globrad}. To simplify somewhat the discussion, we assume that
the eigenvalues satisfy $\lambda_j^{(m)}\leq d j^{-\alpha}$ (with $\alpha>1$) for all $m$ and
concentrate on the rates obtained as a function of the parameters $n,\alpha,M,D$ and $p$, while considering other parameters fixed and hiding them in a big-O notation. Using this simplification, the bound of Theorem~\ref{theorem:excess} reads 
\begin{equation}\label{O-excessrisk}
  \forall t\in[p,2]:\quad \P(\loss_{\hat{f}} - \loss_{f^*}) = O\Big(\big(t^*D\big)^{\frac{2}{1+\alpha}}M^{1+\frac{2}{1+\alpha}\big(\frac{1}{t^*}-1\big)}
n^{-\frac{\alpha}{1+\alpha}}\Big).
\end{equation}
On the other hand, the global Rademacher complexity directly leads to a bound on the supremum of the centered empirical process indexed by $\mathcal F$ and thus also provides a bound on the excess risk (see, e.g., \citealp{BouBouLug04}). Therefore, using Corollary~\ref{cor:globrad}, wherein we upper bound the trace of each $J_m$
by the constant $B$ (and subsume it under the O-notation), we have a second bound on the excess risk of the form
\begin{equation}\label{O-excessrisk-glob}
  \forall t\in[p,2]:\quad \P(\loss_{\hat{f}} - \loss_{f^*}) = O\Big(t^*D M^{\frac{1}{t^*}}
n^{-\frac{1}{2}}\Big).
\end{equation}
First consider the case where $p\geq (\log M)^*$, that is, the best choice in~\eqref{O-excessrisk} and~\eqref{O-excessrisk-glob} is $t=p$.
Clearly, if we hold all other parameters fixed and let $n$ grow to infinity, the rate obtained through the local Rademacher analysis is better
since $\alpha>1$. However, it is also of interest to consider what happens when the number of kernels $M$ and the $\ell_p$ ball radius $D$ 
can grow with $n$. In general, we have a bound on the excess risk given by
the minimum of~\eqref{O-excessrisk} and~\eqref{O-excessrisk-glob}; a straightforward calculation shows that the local Rademacher analysis
improves over the global one whenever
\[
\frac{M^{\frac{1}{p}}}{D} = O( \sqrt{n}).
\]
Interestingly, we note that this ``phase transition'' does not depend on $\alpha$ (i.e. the ``complexity'' of the individual kernels), but only on $p$.

If $p\leq (\log M)^*$, the  best choice in~\eqref{O-excessrisk} and~\eqref{O-excessrisk-glob} is $t=(\log M)^*$. In this case taking the minimum of the two bounds reads
\begin{equation}
\label{eq:minpsmall}
  \forall p \leq (\log M)^*:\quad \P(\loss_{\hat{f}} - \loss_{f^*}) \leq O\Big(\min( D (\log M) n^{-\frac{1}{2}}, 
\big(D \log M\big)^{\frac{2}{1+\alpha}}M^{\frac{\alpha-1}{1+\alpha}}
n^{-\frac{\alpha}{1+\alpha}})\Big),
\end{equation}
and the phase transition when the local Rademacher bound improves
over the global one occurs for
\[
\frac{M}{D\log M} = O( \sqrt{n}).
\]
Finally, it is also interesting to observe the behavior of~\eqref{O-excessrisk} and \eqref{O-excessrisk-glob} as $\alpha \rightarrow \infty$. In this case, it means that only one eigenvalue is nonzero for each kernel, that is, each kernel space is one-dimensional. In other words, in this case we are in the case of ``classical'' aggregation of $M$ basis functions, and the minimum of the two bounds reads
\begin{equation}
\label{eq:bound-onedim}
\forall t \in [p,2]:\quad \P(\loss_{\hat{f}} - \loss_{f^*}) 
\leq O\Big(\min(Mn^{-1}, t^*D M^{\frac{1}{t^*}}
n^{-\frac{1}{2}}\Big).
\end{equation}
In this configuration, observe that the local Rademacher bound is $O(M/n)$ and does not depend on $D$, nor $p$, any longer; in fact, it is the same bound that one would obtain for the empirical risk minimization over the space of all linear combinations of the $M$ base functions, without any restriction on the norm of the coefficients---the $\ell_p$-norm constraint becomes void. The global Rademacher bound on the other hand, still depends crucially on the $\ell_p$ norm constraint. This 
situation is to be compared to the sharp analysis of the optimal  convergence rate of convex aggregation of $M$ functions obtained by \cite{Tsy03} in the framework of squared error loss regression, which are shown to be 
\[
O\paren{\min\paren{\frac{M}{n}, \sqrt{\frac{1}{n} \log \paren{\frac{M}{\sqrt{n}}}}}}\,.
\]
This corresponds to the setting studied here with $D=1,p=1$ and $\alpha\rightarrow \infty$, and we see that the bound \eqref{eq:minpsmall} recovers (up to log factors) in this case
this sharp bound and the related phase transition phenomenon.


\subsection{Discussion of Assumption {\bf (U)} }\label{sec:uass}

Assumption {\bf (U)} is arguably quite a strong hypothesis for the validity of our results (needed for $1\leq p \leq 2$), which was not required for the global Rademacher bound. A similar assumption was made in the
recent work of \cite{RasWaiYu10}, where a related MKL algorithm using an $\ell_1$-type penalty is studied, and bounds are derived that
depend on the ``sparsity pattern'' of the Bayes function, i.e. how many coefficients $w^*_m$ are non-zero. If the kernel spaces are one-dimensional, in which case $\ell_1$-penalized MKL reduces qualitatively to standard lasso-type methods, this assumption can be seen as a strong form of the so-called Restricted Isometry Property (RIP), which is known to be necessary to grant the validity of bounds taking into account the sparsity pattern of the Bayes function.

In the present work, our analysis stays deliberately ``agnostic'' (or worst-case) with respect to the true sparsity pattern (in part because experimental evidence seems to point towards the fact that the Bayes function is not strongly sparse), correspondingly it could legitimately be hoped that the RIP condition, or Assumption {\bf (U)}, could be substantially relaxed. Considering again the special case of one-dimensional kernel spaces and the discussion about the qualitatively equivalent case $\alpha\rightarrow\infty$ in the previous section, it can be seen that Assumption {\bf (U)} is indeed 
unnecessary for bound \eqref{eq:bound-onedim} to hold, and more specifically for the rate of $M/n$ obtained through local Rademacher analysis in this case. However, as we discussed, what happens in this specific case is that the local Rademacher analysis becomes oblivious to the $\ell_p$-norm constraint, and we are left with the
standard parametric convergence rate in dimension $M$. In other words, 
with one-dimensional kernel spaces, the two constraints (on the $L^2(P)$-norm of the function and on the $\ell_p$ block-norm of the coefficients) appearing in the definition of local Rademacher complexity are essentially not active simultaneously. Unfortunately, it is clear that this property is not true anymore for kernels of higher complexity (i.e. with a non-trivial decay rate of the eigenvalues). This is a specificity of the kernel setting as 
compared to combinations of a dictionary of $M$ simple functions, and Assumption {\bf (U)} was in effect used to ``align'' the two constraints. To sum up, Assumption {\bf (U)} is used here for a different purpose from that of the RIP in sparsity analyses of $\ell_1$ regularization methods; it is not clear to us at this point
if this assumption is necessary or if uncorrelated variables $\x^{(m)}$ constitutes a ``worst case'' for our analysis.
We did not suceed so far in relinquishing this assumption for $p\leq 2$, and this question remains open.

Up to our knowledge, there is no previous existing analysis of the $\ell_p$-MKL setting for $p>1$; the recent works of \cite{RasWaiYu10}
and \cite{KolYua10} focus on the case $p=1$ and on the sparsity pattern of the Bayes function. A refined analysis
of $\ell_p$-regularized methods in the case of combination of $M$ basis functions was laid out by \cite{Kol09}, also taking into account the possible soft sparsity pattern of the Bayes function.
Extending the ideas underlying the latter analysis into the kernel setting is likely to open interesting developments.

\subsection{Analysis of the Impact of the Norm Parameter $p$ on the Accuracy of $\ell_p$-norm MKL}\label{sec:sparsity}

As outlined in the introduction, there is empirical evidence that the performance of $\ell_p$-norm MKL crucially depends on the choice of the norm parameter $p$ (cf. Figure~\ref{fig:tss} in the introduction). The aim of this section is to relate the theoretical analysis presented here to this empirically observed phenomenon. We believe that this phenomenon can be (at least partly) explained on base of our excess risk bound obtained in the last section. To this end we will analyze the dependency of the excess risk bounds on the chosen norm parameter $p$. We will show that the optimal $p$ depends on the geometrical properties of the learning problem and that in general---depending on the true geometry---any $p$ can be optimal. Since our excess risk bound is only formulated for $p\leq 2$, we will limit the analysis to the range $p\in[1,2]$.

To start with, first note that the choice of $p$ only affects the excess risk bound in the factor (cf. Theorem~\ref{theorem:excess}
and Equation~\eqref{O-excessrisk})
$$ \nu_t:= \min_{t\in[p,2]}\big(D_pt^*\big)^{\frac{2}{1+\alpha}} M^{1+\frac{2}{1+\alpha}\big(\frac{1}{t^*}-1\big)} .$$
So we write the excess risk as $\P(\loss_{\hat{f}} - \loss_{f^*})=O(\nu_t)$ and hide all variables and constants in the O-notation for the whole section (in particular the sample size $n$ is considered a constant for the purposes of the present discussion). 
It might surprise the reader that we consider the term in $D$ in the bound although it seems from the bound that it does not depend on $p$.  This stems from a subtle reason that we have ignored in this analysis so far: $D$ is related to the approximation properties of the class, i.e., its ability to attain the Bayes hypothesis. For a ``fair'' analysis we should take the approximation properties of the class into account. 

\begin{figure}[h]
  \centering
  \subfigure[~$\beta=2$]{\includegraphics[width=0.32\textwidth]{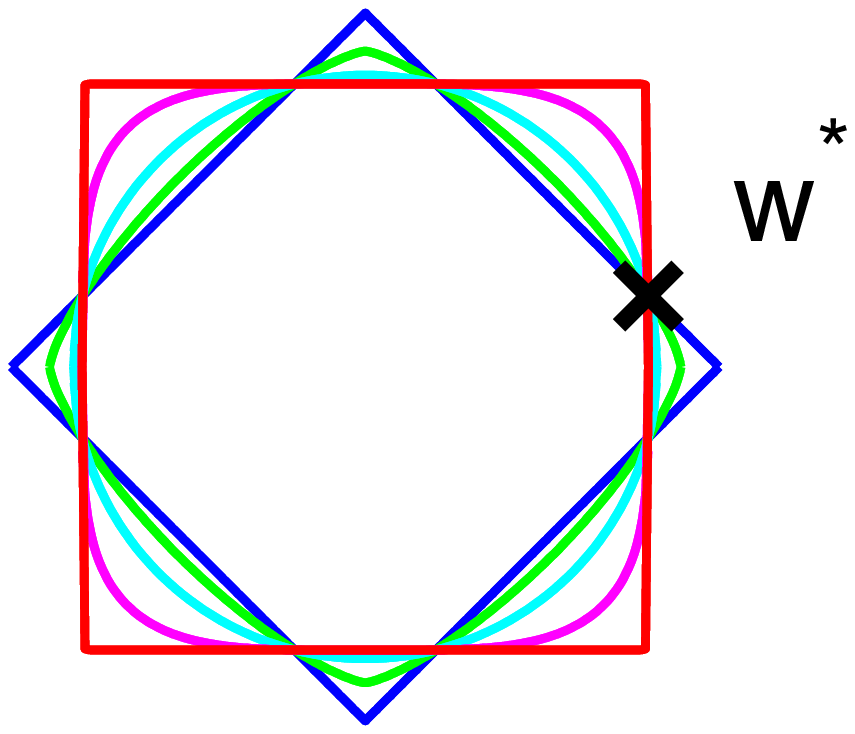}}
  \subfigure[~$\beta=1$]{\includegraphics[width=0.32\textwidth]{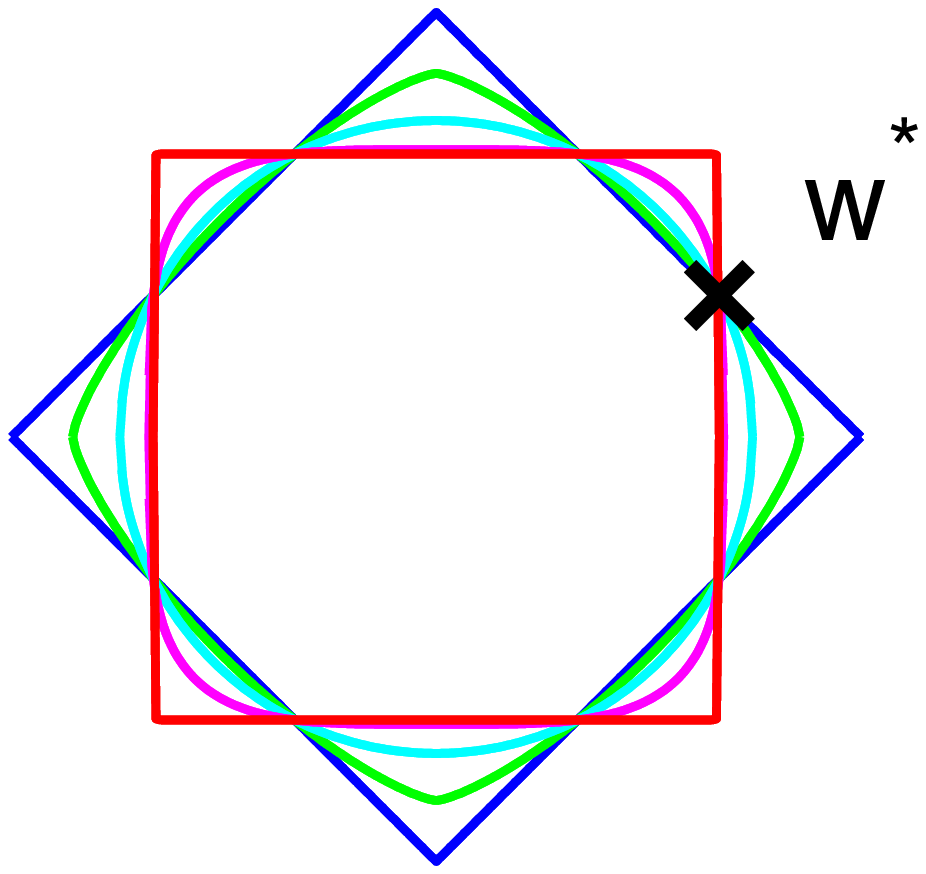}}
  \subfigure[~$\beta=0.5$]{\includegraphics[width=0.32\textwidth]{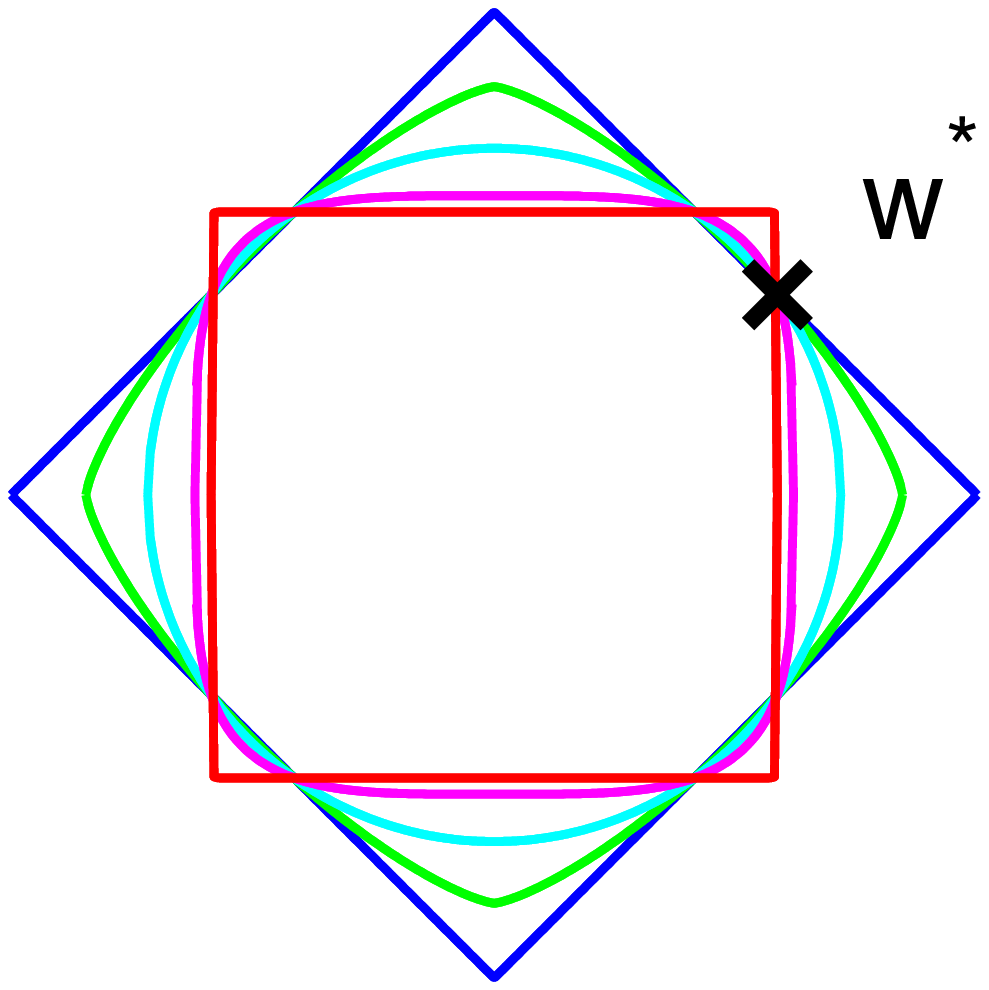}}
  \caption{\label{fig:ex}
      Two-dimensional illustration of the three analyzed learning scenarios, which differ in the soft sparsity of the Bayes hypothesis $\w^*$ ~ (parametrized by $\beta$). {\sc Left:} A soft sparse $\w^*$. {\sc Center:} An intermediate non-sparse $\w^*$.  {\sc Right:} An almost-uniformly non-sparse $\w^*$.}
\end{figure}

To illustrate this, let us assume that the  Bayes hypothesis belongs
to the space $\cH$ and can be represented by $\w^*$; assume further that the block components satisfy $\Vert\w^*_m\Vert_2=m^{-\beta}$, $m=1,\ldots,M$, where $\beta\geq 0$ is a parameter parameterizing the ``soft sparsity'' of the components. For example, the cases $\beta\in\{0.5,1,2\}$ are shown in Figure~\ref{fig:ex} for $M=2$ and assuming that each kernel has rank $1$ (thus being isomorphic to $\mathbb R$). If $n$ is large, the best bias-complexity tradeoff
for a fixed $p$ will correspond to a vanishing bias, so that the best
choice of $D$ will be close to the minimal value such that $\w^* \in \H_{p,D}$, that is,  $D_p=||\w^*||_p$. Plugging in this value for $D_p$,
the bound factor $\nu_p$ becomes
$$ \nu_p:= \Vert\w^*\Vert_p^{\frac{2}{1+\alpha}} \min_{t\in[p,2]}{t^*}^{\frac{2}{1+\alpha}} M^{1+\frac{2}{1+\alpha}\big(\frac{1}{t^*}-1\big)} ~.$$

We can now plot the value $\nu_p$ as a function of $p$ for special choices of $\alpha$, $M$, and $\beta$.
We realized this simulation for $\alpha=2$, $M=1000$, and $\beta\in\{0.5,1,2\}$, which means we generated three learning scenarios with different levels of soft sparsity parametrized by $\beta$. The results are shown in Figure~\ref{fig_p_bound}.  Note that the soft sparsity of  $\w^*$ is increased from the left hand to the right hand side. 
We observe that in the ``soft sparsest'' scenario ($\beta=2$, shown on the left-hand side) the minimum is attained for a quite small $p=1.2$, while for the intermediate case  ($\beta=1$, shown at the center) $p=1.4$ is optimal, and finally in the uniformly non-sparse scenario ($\beta=2$, shown on the right-hand side) the choice of $p=2$ is optimal (although even a higher $p$ could be optimal, but our bound is only valid for $p\in[1,2]$). 

This means that if the true Bayes hypothesis has an intermediately dense representation, our bound gives the strongest generalization guarantees to $\ell_p$-norm MKL using an intermediate choice of $p$. This is also intuitive: if the truth exhibits some soft sparsity but is not strongly sparse, we expect non-sparse MKL to perform better than strongly sparse MKL or the unweighted-sum kernel SVM.

\begin{figure}[t]
  \centering
  \subfigure[~$\beta=2$]{\includegraphics[angle=270,width=0.32\textwidth]{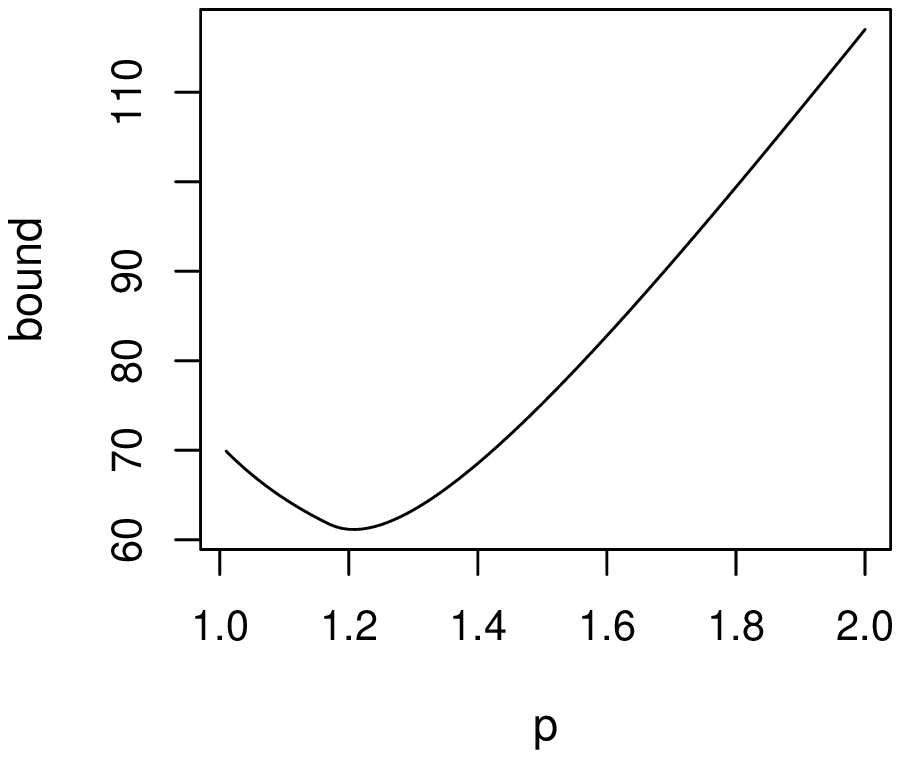}}
  \subfigure[~$\beta=1$]{\includegraphics[angle=270,width=0.32\textwidth]{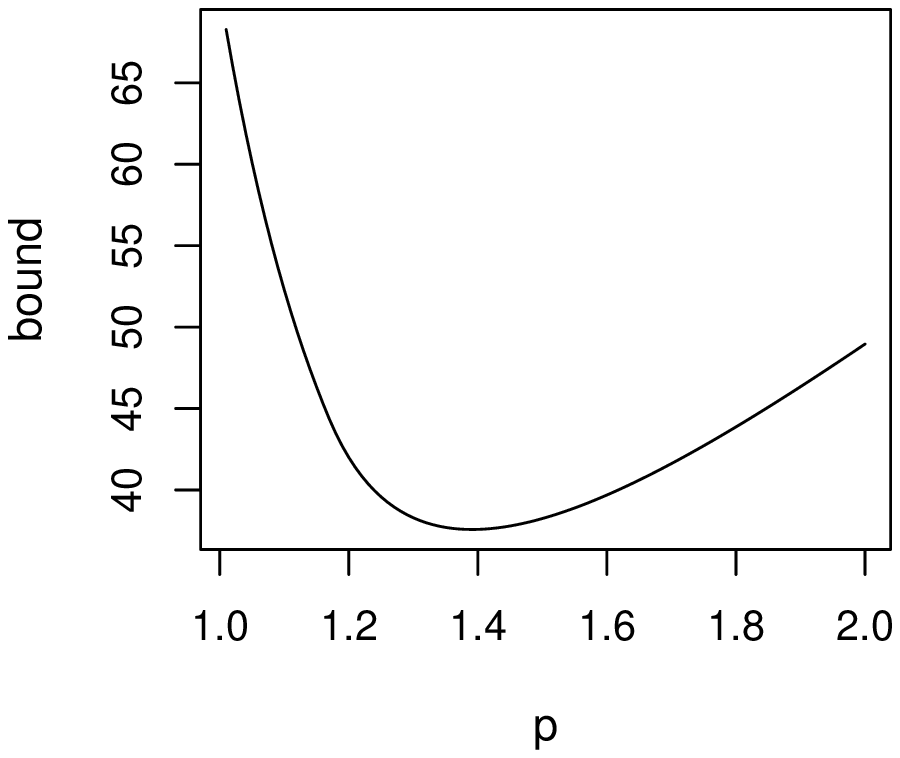}}
  \subfigure[~$\beta=0.5$]{\includegraphics[angle=270,width=0.32\textwidth]{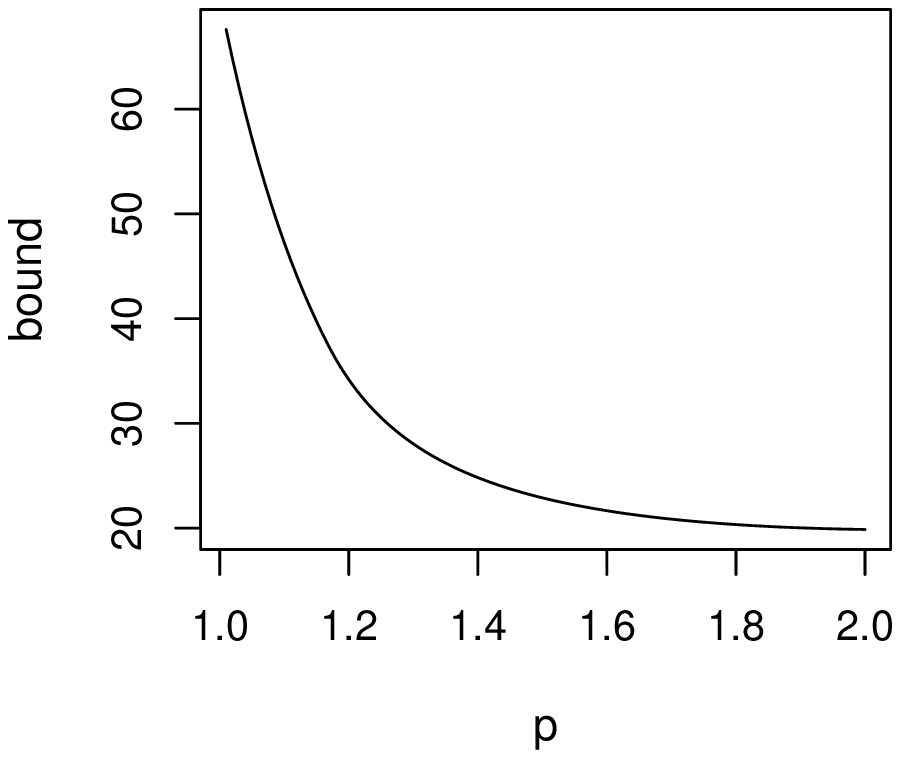}}
  \caption{\label{fig_p_bound}
      Results of the simulation for the three analyzed learning scenarios. The value of the bound factor $\nu_t$ is plotted as a function of $p$. The minimum is attained depending on the true soft sparsity of the Bayes hypothesis $\w^*$  (parametrized by $\beta$). {\sc Left:} An ``almost sparse'' $\w^*$.}
\end{figure}

\section{Conclusion}

We derived a sharp upper bound on the local Rademacher complexity of $\ell_p$-norm multiple kernel learning under the assumption of uncorrelated kernels. We also proved a lower bound that matches the upper one and shows that our result is tight. Using the local Rademacher complexity bound, we derived an excess risk bound that attains the fast rate of $O(n^{-\frac{\alpha}{1+\alpha}})$, where $\alpha$ is the mininum eigenvalue decay rate of the individual kernels.

In a practical case study, we found that the optimal value of that bound depends on the true Bayes-optimal kernel weights. If the true weights exhibit soft sparsity but are not strongly sparse, then the generalization bound is minimized for an intermediate $p$. This is not only intuitive but also supports empirical studies showing that sparse MKL ($p=1$) rarely works in practice, while some intermediate choice of $p$ can improve performance. 

Of course, this connection is only valid if the optimal kernel weights are likely to be non-sparse in practice. Indeed, related research points in that direction. For example, already weak connectivity in a causal graphical model may be sufficient for all variables to be required for optimal predictions, and even the prevalence of sparsity in causal flows is being questioned (e.g., for the social sciences \citealp{Gelman10}, argues that ``There are (almost) no true zeros'').

Finally, we note that seems to be a certain preference for sparse models in the scientific community. However, sparsity by itself should not be considered the ultimate virtue to be strived for---on the contrary: previous MKL research has shown that non-sparse models may improve quite impressively over sparse ones in practical applications. The present analysis supports this by showing  that the reason for this might be traced back to non-sparse MKL attaining better generalization bounds in non-sparse learning scenarios. We remark that this point of view is also supported by related analyses.

For example,  it was shown by \cite{LeePoe08} in a fixed design setup that any sparse estimator (i.e., satisfying the oracle property of correctly predicting the zero 
values of the true target $\w^*$) has a maximal scaled mean squared error (MSMSE) that diverges to $\infty$. This is somewhat suboptimal since, for example, 
least-squares regression has a converging MSMSE. Although this is an asymptotic result, it might also be one of the reasons for finding excellent (nonasymptotic) results in
non-sparse MKL. In another, recent study of \cite{stable}, it was shown that no sparse algorithm can be algorithmically stable. This
is noticeable because algorithmic stability is connected with generalization error \citep{BousquetStab}.

\begin{acks}
We thank Peter L. Bartlett and Klaus-Robert M\"uller for helpful comments on the manuscript.
\end{acks}

\appendix

\section{Lemmata and Proofs}\label{app:lemmata}

\noindent The following result gives a block-structured version of H\"older's inequality \citep[e.g.,][]{SteeleBook}.
\begin{lemma}[Block-structured H\"older inequality]\label{prop-hoelder}
Let $\x=(\x^{(1)},\ldots,\x^{(m)}),~\y=(\y^{(1)},\ldots,\y^{(m)})\in\Hilb=\Hilb_1\times\cdots\times\Hilb_M$. 
Then, for any $p\geq 1$, it holds
$$\langle\x,\y\rangle\leq\Vert\x\Vert_{2,p}\Vert\y\Vert_{2,p^*}~.$$
\end{lemma}
\smallskip
\begin{proof}
By the Cauchy-Schwarz inequality (C.-S.), we have for all $\x,\y\in\mathcal H$:
\begin{eqnarray*}
  \langle\x,\y\rangle &=& \sum_{m=1}^M \langle\x^{(m)},\y^{(m)}\rangle ~\stackrel{\text{C.-S.}}{\leq}~ \sum_{m=1}^M \Vert\x\Vert_2\Vert\y\Vert_2 \\
     &=& \big\langle(\Vert\x^{(1)}\Vert_2,\ldots,\Vert\x^{(M)}\Vert_2),(\Vert\y^{(1)}\Vert_2,\ldots,\Vert\y^{(M)}\Vert_2)\big\rangle.\\
     &\stackrel{\text{H\"older}}{\leq}& \Vert\x\Vert_{2,p}\Vert\y\Vert_{2,p^*}
\end{eqnarray*}
\end{proof}

\begin{proof}\textbf{of Lemma~\ref{lemma:rosen} (Rosenthal + Young)}
It is clear that the result trivially holds for $\frac{1}{2}\leq p\leq 1$ with $C_q=1$ by Jensen's inequality . In the case $p\geq 1$,
we apply Rosenthal's inequality \citep{Ros70} to the sequence $X_1,\ldots,X_n$ thereby using the optimal constants computed in \cite{Ibra01}, that are, $C_q=2$ ($q\leq2$) and $C_q=\E Z^q$ ($q\geq2$), respectively, where $Z$ is a random variable distributed according to a Poisson law with parameter $\lambda=1$. This yields 
\begin{equation}\label{eq:rosen}
  \E\bigg(\frac{1}{n}\sum_{i=1}^nX_i\bigg)^q \leq C_q \max\left(\frac{1}{n^q}\sum_{i=1}^n\E X_i^q,\left(\frac{1}{n}\sum_{i=1}^n X_i\right)^q\right).
\end{equation}
By using that $X_i\leq B$ holds almost surely, we could readily obtain a bound of the form $\frac{B^q}{n^{q-1}}$ on the first term. However, this is loose and for $q=1$ does not converge to zero when $n\rightarrow\infty$. Therefore, we follow a different approach based on Young's inequality \citep[e.g.][]{SteeleBook}:
\begin{eqnarray*}
  \frac{1}{n^q}\sum_{i=1}^n\E X_i^q &\leq& \bigg(\frac{B}{n}\bigg)^{q-1}\frac{1}{n}\sum_{i=1}^n\E X_i \\
  &\stackrel{\text{Young}}{\leq}& \frac{1}{q^*} \left(\frac{B}{n}\right)^{q^*(q-1)} + \frac{1}{q}\left(\frac{1}{n}\sum_{i=1}^n\E X_i\right)^q\\
  &=& \frac{1}{q^*} \left(\frac{B}{n}\right)^{q} + \frac{1}{q}\left(\frac{1}{n}\sum_{i=1}^n\E X_i\right)^q .
\end{eqnarray*}
It thus follows from \eqref{eq:rosen} that for all $q\geq \frac{1}{2}$
$$\E\bigg(\frac{1}{n}\sum_{i=1}^nX_i\bigg)^q \leq C_q  \left(\Big(\frac{B}{n}\Big)^{q} + \Big(\frac{1}{n}\sum_{i=1}^n\E X_i\Big)^q \right) ,$$
where $C_q$ can be taken as $2$ ($q\leq 2$) and $\E Z^q$ ($q\geq 2$), respectively, where $Z$ is Poisson-distributed. In the subsequent Lemma~\ref{lemma:poisson} we show $\E Z^q\leq (q+e)^q$. Clearly, for $q\geq\frac{1}{2}$ it holds $q+e\leq qe + e q = 2eq$ so that in any case $C_q\leq\max(2,2eq)\leq 2eq$, which concludes the result.
\end{proof}

\noindent We use the following Lemma gives a handle on the $q$-th moment of a Poisson-distributed random variable and is used in the previous Lemma.

\begin{lemma}\label{lemma:poisson}
For the $q$-moment of a  random variable $Z$ distributed according to a Poisson law with parameter $\lambda=1$, the following inequality holds for all $q\geq 1$:
$$ \E Z^q\stackrel{\text{def.}}{=}\frac{1}{e}\sum_{k=0}^\infty \frac{k^q}{k!} \leq (q+e)^q.$$
\end{lemma}

\begin{proof}
We start by decomposing $\E Z^q$ as follows:
\begin{eqnarray}\label{eq:poisson}
  \E^q&=&\frac{1}{e}\left(0+\sum_{k=1}^q\frac{k^q}{k!}+\sum_{k=q+1}^\infty\frac{k^q}{k!}\right) \nonumber\\
  &=& \frac{1}{e}\left(\sum_{k=1}^q\frac{k^{q-1}}{(k-1)!}+\sum_{k=q+1}^\infty\frac{k^q}{k!}\right) \nonumber\\
  &\leq& \frac{1}{e}\left(q^q+\sum_{k=q+1}^\infty\frac{k^q}{k!}\right) \\
\end{eqnarray}
Note that by Stirling's approximation it holds $k!=\sqrt{2\pi}e^{\tau_k}k\left(\frac{k}{e}\right)^q$ with $\frac{1}{12k+1}<\tau_k<\frac{1}{12k}$ for all $q$. Thus
\begin{eqnarray*}
  \sum_{k=q+1}^\infty \frac{k^q}{k!} &=& \sum_{k=q+1}^\infty \frac{1}{\sqrt{2\pi}e^{\tau_k}k}e^kk^{-(k-q)} \\
  &=& \sum_{k=1}^\infty \frac{1}{\sqrt{2\pi}e^{\tau_{k+q}}(k+q)}e^{k+q}k^{-k} \\
  &=&  e^q \sum_{k=1}^\infty\frac{1}{\sqrt{2\pi}e^{\tau_{k+q}}(k+q)} \left(\frac{e}{k}\right)^k\\
  &\stackrel{(*)}{\leq}&  e^q\sum_{k=1}^\infty \frac{1}{\sqrt{2\pi}e^{\tau_{k}}k} \left(\frac{e}{k}\right)^k\\
  &\stackrel{\text{Stirling}}{=}&  e^q\sum_{k=1}^\infty \frac{1}{k!}\\
  &=& e^{q+1}
\end{eqnarray*}
where for $(*)$ note that $e^{\tau_{k}}k\leq e^{\tau_{k+q}}(k+q)$ can be shown by some algebra using $\frac{1}{12k+1}<\tau_k<\frac{1}{12k}$. 
Now by \eqref{eq:poisson}
$$\E Z^q = \frac{1}{e}\left(q^q+e^{q+1}\right) \leq q^q + e^q \leq (q+e)^q,$$
which was to show.
\end{proof}

\begin{lemma}\label{lemma:combin}
For any $\a,\b\in\mathbb R^m_+$ it holds for all $q\geq 1$

$$ \norm{\a}_q+\norm{\b}_q\leq 2^{1-{\frac{1}{q}}}\norm{\a+\b}_q\leq 2\norm{\a+\b}_q .$$
\end{lemma}

\begin{proof}
Let $\a=(a_1,\ldots,a_m)$ and $\b=(b_1,\ldots,b_m)$. Because all components of
$\a,\b$ are nonnegative, we have 
$$ \forall i=1,\ldots,m: ~ a_i^q+b_i^q\leq \big(a_i+b_i\big)^q$$
and thus 
\begin{equation}\label{eq:auxAB}
  \norm{\a}_q^q+\norm{\b}_q^q\leq\norm{\a+\b}^q_q .
\end{equation}
We conclude by $\ell_q$-to-$\ell_1$ conversion (see \eqref{eq:conv})
\begin{eqnarray*}
   \norm{\a}_q+\norm{\b}_q &=& \big\Vert\big(\norm{\a}_q,\norm{\b}_q\big)\big\Vert_1 ~ \stackrel{\eqref{eq:conv}}{\leq} ~ 2^{1-{\frac{1}{q}}}
     \big\Vert\big(\norm{\a}_q,\norm{\b}_q\big)\big\Vert_q ~ = ~ 2^{1-{\frac{1}{q}}} \big(\norm{\a}_q^q+\norm{\b}^q_q\big)^{\frac{1}{q}} \\
     &\stackrel{\eqref{eq:auxAB}}{\leq}&  2^{1-{\frac{1}{q}}} \norm{\a+\b}_q, 
\end{eqnarray*}
which completes the proof.
\end{proof}

\bibliography{mkl}

\begin{thebibliography}{33}
\providecommand{\natexlab}[1]{#1}
\providecommand{\url}[1]{\texttt{#1}}
\expandafter\ifx\csname urlstyle\endcsname\relax
  \providecommand{\doi}[1]{doi: #1}\else
  \providecommand{\doi}{doi: \begingroup \urlstyle{rm}\Url}\fi

\bibitem[Bach et~al.(2004)Bach, Lanckriet, and Jordan]{BacLanJor04}
F.~R. Bach, G.~R.~G. Lanckriet, and M.~I. Jordan.
\newblock Multiple kernel learning, conic duality, and the {SMO} algorithm.
\newblock In \emph{Proc.~21st ICML}. ACM, 2004.

\bibitem[Bartlett and Mendelson(2002)]{BarMen02}
P.~Bartlett and S.~Mendelson.
\newblock Rademacher and gaussian complexities: Risk bounds and structural
  results.
\newblock \emph{Journal of Machine Learning Research}, 3:\penalty0 463--482,
  Nov. 2002.

\bibitem[Bartlett et~al.(2005)Bartlett, Bousquet, and Mendelson]{BarBouMen05}
P.~L. Bartlett, O.~Bousquet, and S.~Mendelson.
\newblock {Local Rademacher complexities}.
\newblock \emph{Annals of Statistics}, 33\penalty0 (4):\penalty0 1497--1537,
  2005.

\bibitem[Bouckaert et~al.(2010)Bouckaert, Frank, Hall, Holmes, Pfahringer,
  Reutemann, and Witten]{Weka}
R.~R. Bouckaert, E.~Frank, M.~A. Hall, G.~Holmes, B.~Pfahringer, P.~Reutemann,
  and I.~H. Witten.
\newblock {WEKA}--experiences with a java open-source project.
\newblock \emph{Journal of Machine Learning Research}, 11:\penalty0 2533--2541,
  2010.

\bibitem[Bousquet and Elisseeff(2002)]{BousquetStab}
O.~Bousquet and A.~Elisseeff.
\newblock Stability and generalization.
\newblock \emph{J. Mach. Learn. Res.}, 2:\penalty0 499--526, March 2002.
\newblock ISSN 1532-4435.

\bibitem[Bousquet et~al.(2004)Bousquet, Boucheron, and Lugosi]{BouBouLug04}
O.~Bousquet, S.~Boucheron, and G.~Lugosi.
\newblock Introduction to statistical learning theory.
\newblock In O.~Bousquet, U.~von Luxburg, and G.~R\"atsch, editors,
  \emph{Advanced Lectures on Machine Learning}, volume 3176 of \emph{Lecture
  Notes in Computer Science}, pages 169--207. Springer Berlin / Heidelberg,
  2004.

\bibitem[Cortes(2009)]{Cortes2009}
C.~Cortes.
\newblock Invited talk: Can learning kernels help performance?
\newblock In \emph{Proceedings of the 26th Annual International Conference on
  Machine Learning}, ICML '09, pages 1:1--1:1, New York, NY, USA, 2009. ACM.
\newblock ISBN 978-1-60558-516-1.

\bibitem[Cortes et~al.(2008)Cortes, Gretton, Lanckriet, Mohri, and
  Rostamizadeh]{WSNips}
C.~Cortes, A.~Gretton, G.~Lanckriet, M.~Mohri, and A.~Rostamizadeh.
\newblock Proceedings of the {NIPS} {W}orkshop on {K}ernel {L}earning:
  {A}utomatic {S}election of {O}ptimal {K}ernels, 2008.
\newblock URL \url{http://www.cs.nyu.edu/learning_kernels}.

\bibitem[Cortes et~al.(2010)Cortes, Mohri, and Rostamizadeh]{CorMohRos10}
C.~Cortes, M.~Mohri, and A.~Rostamizadeh.
\newblock Generalization bounds for learning kernels.
\newblock In \emph{Proceedings, 27th ICML}, 2010.

\bibitem[Gehler and Nowozin(2009)]{GehNow09}
P.~V. Gehler and S.~Nowozin.
\newblock Let the kernel figure it out: Principled learning of pre-processing
  for kernel classifiers.
\newblock In \emph{IEEE Computer Society Conference on Computer Vision and
  Pattern Recognition}, 06 2009.

\bibitem[Gelman(2010)]{Gelman10}
A.~Gelman.
\newblock Causality and statistical learning.
\newblock \emph{American Journal of Sociology}, 0, 2010.

\bibitem[Ibragimov and Sharakhmetov(2001)]{Ibra01}
R.~Ibragimov and S.~Sharakhmetov.
\newblock The best constant in the rosenthal inequality for nonnegative random
  variables.
\newblock \emph{Statistics \& Probability Letters}, 55\penalty0 (4):\penalty0
  367 -- 376, 2001.
\newblock ISSN 0167-7152.

\bibitem[Kahane(1985)]{Kahane85}
J.-P. Kahane.
\newblock \emph{Some random series of functions}.
\newblock Cambridge University Press, 2nd edition, 1985.

\bibitem[Kloft et~al.(2009)Kloft, Brefeld, Sonnenburg, Laskov, M\"{u}ller, and
  Zien]{KloBreSonZieLasMue09}
M.~Kloft, U.~Brefeld, S.~Sonnenburg, P.~Laskov, K.-R. M\"{u}ller, and A.~Zien.
\newblock Efficient and accurate lp-norm multiple kernel learning.
\newblock In Y.~Bengio, D.~Schuurmans, J.~Lafferty, C.~K.~I. Williams, and
  A.~Culotta, editors, \emph{Advances in Neural Information Processing Systems
  22}, pages 997--1005. MIT Press, 2009.

\bibitem[Kloft et~al.(2011)Kloft, Brefeld, Sonnenburg, and
  Zien]{KloBreSonZie2011}
M.~Kloft, U.~Brefeld, S.~Sonnenburg, and A.~Zien.
\newblock $\ell_p$-norm multiple kernel learning.
\newblock \emph{Journal of Machine Learning Research}, 2011.
\newblock To appear. {URL} \url{http://doc.ml.tu-berlin.de/nonsparse_mkl/}.

\bibitem[Koltchinskii(2001)]{Kol01}
V.~Koltchinskii.
\newblock Rademacher penalties and structural risk minimization.
\newblock \emph{IEEE Transactions on Information Theory}, 47\penalty0
  (5):\penalty0 1902--1914, 2001.

\bibitem[Koltchinskii(2006)]{Kol06}
V.~Koltchinskii.
\newblock {Local Rademacher complexities and oracle inequalities in risk
  minimization}.
\newblock \emph{Annals of Statistics}, 34\penalty0 (6):\penalty0 2593--2656,
  2006.

\bibitem[Koltchinskii(2009)]{Kol09}
V.~Koltchinskii.
\newblock Sparsity in penalized empirical risk minimization.
\newblock \emph{Annales de l'Institut Henri Poincar\'{e}, Probabilit\'{e}s et
  Statistiques}, 45\penalty0 (1):\penalty0 7--57, 2009.

\bibitem[Koltchinskii and Yuan(2010)]{KolYua10}
V.~Koltchinskii and M.~Yuan.
\newblock Sparsity in multiple kernel learning.
\newblock \emph{Annals of Statistics}, 38\penalty0 (6):\penalty0 3660--3695,
  2010.

\bibitem[Kwapi\'en and Woyczy\'nski(1992)]{KwaWoy92}
S.~Kwapi\'en and W.~A. Woyczy\'nski.
\newblock \emph{Random Series and Stochastic Integrals: Single and Multiple}.
\newblock Birkh\"auser, Basel and Boston, M.A., 1992.

\bibitem[Lanckriet et~al.(2004)Lanckriet, Cristianini, Ghaoui, Bartlett, and
  Jordan]{LanCriGhaBarJor04}
G.~Lanckriet, N.~Cristianini, L.~E. Ghaoui, P.~Bartlett, and M.~I. Jordan.
\newblock Learning the kernel matrix with semi-definite programming.
\newblock \emph{JMLR}, 5:\penalty0 27--72, 2004.

\bibitem[Leeb and P\"otscher(2008)]{LeePoe08}
H.~Leeb and B.~M. P\"otscher.
\newblock Sparse estimators and the oracle property, or the return of {H}odges'
  estimator.
\newblock \emph{Journal of Econometrics}, 142:\penalty0 201--211, 2008.

\bibitem[Mendelson(2003)]{Men2003}
S.~Mendelson.
\newblock On the performance of kernel classes.
\newblock \emph{J. Mach. Learn. Res.}, 4:\penalty0 759--771, December 2003.

\bibitem[Micchelli and Pontil(2005)]{MicPon05}
C.~A. Micchelli and M.~Pontil.
\newblock Learning the kernel function via regularization.
\newblock \emph{Journal of Machine Learning Research}, 6:\penalty0 1099--1125,
  2005.

\bibitem[Raskutti et~al.(2010)Raskutti, Wainwright, and Yu]{RasWaiYu10}
G.~Raskutti, M.~J. Wainwright, and B.~Yu.
\newblock Minimax-optimal rates for sparse additive models over kernel classes
  via convex programming.
\newblock \emph{CoRR}, abs/1008.3654, 2010.

\bibitem[Rosenthal(1970)]{Ros70}
H.~Rosenthal.
\newblock On the subspaces of {L}$_p$ ($p>2$) spanned by sequences of
  independent random variables.
\newblock \emph{Israel J. Math.}, 8:\penalty0 273--303, 1970.

\bibitem[Searle(1980)]{searle}
J.~R. Searle.
\newblock Minds, brains, and programs.
\newblock \emph{Behavioral and Brain Sciences}, 3\penalty0 (03):\penalty0
  417--424, 1980.
\newblock \doi{10.1017/S0140525X00005756}.
\newblock URL \url{http://dx.doi.org/10.1017/S0140525X00005756}.

\bibitem[Sonnenburg et~al.(2006)Sonnenburg, R\"atsch, Sch\"afer, and
  Sch\"olkopf]{SonRaeSchSch06}
S.~Sonnenburg, G.~R\"atsch, C.~Sch\"afer, and B.~Sch\"olkopf.
\newblock Large scale multiple kernel learning.
\newblock \emph{Journal of Machine Learning Research}, 7:\penalty0 1531--1565,
  July 2006.

\bibitem[Steele(2004)]{SteeleBook}
J.~M. Steele.
\newblock \emph{The Cauchy-Schwarz Master Class: An Introduction to the Art of
  Mathematical Inequalities}.
\newblock Cambridge University Press, New York, NY, USA, 2004.
\newblock ISBN 052154677X.

\bibitem[Stone(1974)]{Stone1974}
M.~Stone.
\newblock Cross-validatory choice and assessment of statistical predictors
  (with discussion).
\newblock \emph{Journal of the Royal Statistical Society}, B36:\penalty0
  111--147, 1974.

\bibitem[Tsybakov(2003)]{Tsy03}
A.~Tsybakov.
\newblock Optimal rates of aggregation.
\newblock In B.~Sch\"{o}lkopf and M.~Warmuth, editors, \emph{Computational
  Learning Theory and Kernel Machines (COLT-2003)}, volume 2777 of
  \emph{Lecture Notes in Artificial Intelligence}, pages 303--313. Springer,
  2003.

\bibitem[Williamson et~al.(2001)Williamson, Smola, and
  Sch{\"o}lkopf]{WilSmoSch01}
R.~C. Williamson, A.~J. Smola, and B.~Sch{\"o}lkopf.
\newblock Generalization performance of regularization networks and support
  vector machines via entropy numbers of compact operators.
\newblock \emph{IEEE Transactions on Information Theory}, 47\penalty0
  (6):\penalty0 2516--2532, 2001.

\bibitem[Xu et~al.(2008)Xu, Mannor, and Caramanis]{stable}
H.~Xu, S.~Mannor, and C.~Caramanis.
\newblock Sparse algorithms are not stable: A no-free-lunch theorem.
\newblock In \emph{Proceedings of the 46th Annual Allerton Conference on
  Communication, Control, and Computing}, pages 1299 --1303, 2008.

\end{thebibliography}

\end{document}